\documentclass[twoside,11pt]{article}
\usepackage[dvipsnames]{xcolor}
\usepackage{blindtext}
\usepackage[preprint]{jmlr2e}

\newcommand{\ds}[1]{{\color{magenta}[DS: #1]}}
\usepackage{lastpage}
\jmlrheading{23}{2025}{1-\pageref{LastPage}}{1/21; Revised 5/22}{9/22}{21-0000}{Russell Tsuchida, Jiawei Liu, Cheng Soon Ong, Dino Sejdinovic}

\ShortHeadings{Preprint}{Tsuchida, Liu, Ong, Sejdinovic}
\firstpageno{1}

\usepackage{appendix}
\usepackage{graphicx} % Required for inserting images
\usepackage{bm}
\usepackage{amssymb,amsmath,bbm}
\usepackage{upgreek}
\usepackage{natbib}
\usepackage{enumitem}
\newcommand\numberthis{\addtocounter{equation}{1}\tag{\theequation}}
\usepackage{thmtools}
\usepackage{thm-restate}

\usepackage{tikz}
\usetikzlibrary{shapes,arrows}
\usetikzlibrary{arrows,calc,positioning}
\usetikzlibrary{bayesnet}
\usepackage{pgfplots}
\pgfplotsset{compat=newest}

\usepackage[makeroom]{cancel}
\tikzset{
    block/.style = {draw, rectangle,
        minimum height=1cm,
        minimum width=2cm},
    input/.style = {coordinate,node distance=1cm},
    output/.style = {coordinate,node distance=4cm},
    arrow/.style={draw, -latex,node distance=2cm},
    pinstyle/.style = {pin edge={latex-, black,node distance=2cm}},
    sum/.style = {draw, circle, node distance=1cm},
    latent/.append style={minimum size=1cm},
    obs/.append style={minimum size=1cm}
}
\usetikzlibrary{patterns}

\DeclareMathAlphabet{\mathbbold}{U}{bbold}{m}{n}

\DeclareSymbolFont{bbold}{U}{bbold}{m}{n}
\DeclareSymbolFontAlphabet{\mathbbold}{bbold}



\renewcommand{\vector}[1]{\bm{\lowercase{#1}}}
\renewcommand{\matrix}[1]{\bm{\uppercase{#1}}}
\newcommand{\rv}[1]{\mathsf{#1}}

\DeclareMathOperator\sqltwo{SL^2}
\DeclareMathOperator\kl{KL}
\DeclareMathOperator\tv{TV}
\DeclareMathOperator\sqhel{SH}

\DeclareMathOperator{\Tr}{Tr}
\DeclareMathOperator{\Cov}{Cov}

\DeclareMathOperator*{\argmin}{\arg\!\min}
\usepackage{xspace}

\newtheorem{assumption}{Assumption}

\DeclareMathOperator{\rank}{rank}

\usepackage{centernot}
\usepackage{hhline}

\usepackage{makecell}
\usepackage{booktabs}
\usepackage{multicol}
\usepackage{multirow}
\usepackage{pifont}% http://ctan.org/pkg/pifont
\usepackage[T1]{fontenc}

\usepackage[flushleft]{threeparttable}
\usepackage{tabularray}

\begin{document}

%\title{Computational, geometric and statistical properties of singular, even-order monomial and squared families}
%\title{Computational, geometric and statistical properties of squared families and their friends}
%\title{When can a singularity help? Squared families and their friends as alternatives to exponential families}
%\title{When can a singularity help? Squared and friends as alternatives to exponential families}
\title{Squared families: Searching beyond regular probability models}

\author{\name Russell Tsuchida\thanks{Work partially done while at Data61, CSIRO} \email russell.tsuchida@monash.edu \\
       \addr Data Science and AI Group\\
       Monash University \\
       Clayton, VIC 3800, Australia
       \AND
       \name Jiawei Liu \email jiawei.liu3@anu.edu.au \\
       \addr School of Computing\\
       Australian National University \\
       Canberra, ACT 2600, Australia 
       \AND 
       \name Cheng Soon Ong \email chengsoon.ong@anu.edu.au \\
       \addr Data61, CSIRO and 
       Australian National University\\
       Canberra, ACT 2600, Australia
       \AND
       \name Dino Sejdinovic \email dino.sejdinovic@adelaide.edu.au \\
       \addr School of Computer and Mathematical Sciences \\
       The University of Adelaide \\
       Adelaide, SA 5005, Australia}

\editor{My editor}

\maketitle

\begin{abstract}
We introduce squared families, which are families of probability densities obtained by squaring a linear transformation of a statistic. Squared families are singular, however their singularity can easily be handled so that they form regular models. After handling the singularity, squared families possess many convenient properties. 
%These properties are shown in a ``forward'' fashion, by starting with the definition of squared families.
Their Fisher information is a conformal transformation of the Hessian metric induced from a Bregman generator. The Bregman generator is the normalising constant, and yields a statistical divergence on the family. The normalising constant admits a helpful parameter-integral factorisation, meaning that only one parameter-independent integral needs to be computed for all normalising constants in the family, unlike in exponential families. Finally, the squared family kernel is the only integral that needs to be computed for the Fisher information, statistical divergence and normalising constant. 
We then describe how squared families are special in the broader class of $g$-families, which are obtained by applying a sufficiently regular function $g$ to a linear transformation of a statistic. 
%We derive their properties in a ``backwards'' fashion, starting with desiderata and providing characterisations leading to the squared family.
After removing special singularities, positively homogeneous families and exponential families are the only $g$-families for which the Fisher information is a conformal transformation of the Hessian metric, where the generator depends on the parameter only through the normalising constant. Even-order monomial families are the only infinitely differentiable positively homogeneous families, and also admit natural parameter-integral factorisations, unlike exponential families. 
Finally, we study parameter estimation and density estimation in squared families, in the well-specified and misspecified settings. We use a universal approximation property to show that squared families can learn sufficiently well-behaved target densities at an asymptotic rate of $\mathcal{O}(N^{-1/2})+C n^{-1/4}$, where $N$ is the number of datapoints, $n$ is the number of parameters, and $C$ is some data-independent constant.
\end{abstract}


\setcounter{tocdepth}{2}
%\tableofcontents

\section{Introduction}
\subsection{Families of probability distributions and their applications}
Tractable, flexible and learnable families of probability density functions find applications throughout statistics and machine learning.
An extreme but common application is (conditional) density estimation, where one attempts to learn an approximation to a target density $q$ by picking an element of the family which in some sense best matches data sampled from $q$~\citep{barron1991approximation,deisenroth2020mathematics,mclachlan2019finite}. 
Another extreme is parameter estimation, where the target density $q$ is known to belong to the family and one attempts to find identifiable parameters of $q$~\citep{lehmann2006theory}.
In between these two extremes, there exists a rich set of problems such as density ratio estimation~\citep{sugiyama2012density}, divergence estimation, clustering~\citep{banerjee2005clustering-bregman}, generalised linear modelling (including regression and classification)~\citep{mccullagh1989generalized},
parametric two-sample testing~\citep{lehmann1986testing}, and generally, inference and estimation in any graphical model~\citep{wainwright2008graphical}.
Central to all these problems are parametric families of probability densities which have favourable computational, geometric and statistical properties.

\paragraph{Exponential families}
Many of the works cited in the previous paragraph study special cases of (mixtures of) \emph{exponential families}.
Exponential families are ubiquitous in such applications mostly because of their convenient geometric, statistical, and, in some cases, computational properties.

\emph{Geometric properties: } Every exponential family forms a manifold indexed by its natural parameter~\citep{amari2016information-geometry}. 
On this manifold, the Bregman divergence between parameters generated by the strictly convex log normalising constant is equal to the reverse KL divergence between probability distributions.
A dual parameter (the expectation parameter) and divergence is obtained through the convex conjugate of the Bregman generator.
The Riemannian metric is the strictly positive definite Fisher information.

\emph{Statistical properties: } The Fisher information is exactly the precision of parameter estimates in the exponential family, attaining equality in the Cramer-Rao bound (uniquely so, under certain regularity conditions~\citep{EF-CRLB-1973,joshi1976}). 
More generally, beyond the exponential family, the Fisher information describes the precision of the asymptotically Gaussian maximum likelihood estimates around the true value. 
The Fisher information has a particularly nice form in exponential families, as it is the Hessian of the Bregman generator, which is also the covariance of the minimal sufficient statistic under the model~\citep{wainwright2008graphical}.

\emph{Computational properties: }
Exponential families are also the only families in which there is a sufficient statistic whose dimension remains bounded as the size of i.i.d. samples increases (among families in which the domain does not depend on the parameter), a result known as the Pitman-Koopman-Darmois theorem.
This means that parameters can be updated using only a finite dimensional statistic, which can be maintained without storing the full dataset of observations.
This is helpful in graphical models, and particularly attractive in Bayesian settings, where Bayesian updating can sometimes be done in closed-form by updating only the finite dimensional sufficient statistic~\citep{wainwright2008graphical}.
Unfortunately, conjugate priors for exponential family likelihoods must be carefully chosen, often forcing one to use a particular conceptually inappropriate prior.
Furthermore, even in the non-Bayesian setting, computing the normalising constant in general beyond special named cases (such as Gaussian, Laplacian, Poisson, Gamma, Binomial, Rayleigh etc.~\citep{nielsen2009statistical}) requires computing or approximating a non-trivial exponential integral. 

In modern machine learning architectures, maintaining flexibility of models with tractable normalising constants~\citep{pmlr-v51-wilson16,papamakarios2021normalizing}, approximating or otherwise bypassing computation~\citep{lecun2006tutorial,graves2011practical,knoblauch2022optimization} of the normalising constant remains an active area of research, with applications in parameter estimation, inference and prediction.

\begin{table}
\centering
\resizebox{\textwidth}{!}{{\begin{tblr}{
  vline{2} = {-}{},
  hline{2} = {-}{},
}
 & $g$ & Exponential & {Positively \\ homogeneous} & {Even-order \\ monomial} & Squared \\
Normalising \\constant $z(\vector{\theta})$ & $\int_{\mathbb{X}} g\big( \vector{\theta}^\top \vector{\psi}(\vector{x}) \big)\, \mu(d\vector{x})$ & $\int_{\mathbb{X}} \exp\big( \vector{\theta}^\top \vector{\psi}(\vector{x}) \big)\, \mu(d\vector{x})$ & $\int_{\mathbb{X}} g\big( \vector{\theta}^\top \vector{\psi}(\vector{x}) \big)\, \mu(d\vector{x})$ & $\left\langle  \vector{\theta}^{\otimes k},\int_{\mathbb{X}}\vector{\psi}(\vector{x})^{\otimes k}\, \mu(d\vector{x})  \right\rangle $ & $\left\langle \vector{\theta}\vector{\theta}^\top, \matrix{K}_{\mu, \vector{\psi}} \right\rangle = \vector{\theta}^\top \matrix{K}_{\mu, \vector{\psi}} \vector{\theta}$ \\
Bregman \\generator $\phi(\vector{\theta})$ & ? & $\log z(\vector{\theta}) $ & $z(\vector{\theta})$ & $z(\vector{\theta})$ & $z(\vector{\theta})$ \\
Divergence link & ? & $d_\phi(\vector{\theta}_1 : \vector{\theta}_2) = \kl(p_2: p_1)$ & ? & ? & $d_\phi(\vector{\theta}_1 : \vector{\theta}_2) = \sqltwo(\sqrt{p_2}: \sqrt{p_1})$ \\
Metric link & ? & $\matrix{G}(\vector{\theta} ) = \nabla^2 \phi(\vector{\theta}) = \Cov[\vector{\psi}(\vector{\rv{x}})]$ & $\matrix{G}(\vector{\theta}) = \frac{k}{k-1} \frac{\nabla^2 \phi(\vector{\theta})}{z(\vector{\theta})}$ & $\matrix{G}(\vector{\theta}) = \frac{k}{k-1} \frac{\nabla^2 \phi(\vector{\theta})}{z(\vector{\theta})}$ & $\matrix{G}(\vector{\theta}) = 2 \frac{\nabla^2 \phi(\vector{\theta})}{z(\vector{\theta})}$ \\
\end{tblr}}
}
\caption{\label{tab:contrast} Contrasting normalising constants $z$, Bregman generators $\phi$, Bregman-statistical divergence links, and Hessian-Fisher information metric links for various $g$-families (after appropriately handling singularity through dimension augmentation --- see \S~\ref{sec:squared_singular} and~\ref{sec:g_singular}) of densities of the form $p(\vector{x} \mid \vector{\theta}) \mu(d\vector{x}) \propto g\big( \vector{\theta}^\top \vector{\psi}(\vector{x}) \big) \mu(d\vector{x})$. 
Here $\matrix{G}$ denotes the Fisher information and $d_\phi$ denotes the Bregman divergenced generated by $\phi$.
Positively homogeneous families have order $k \geq 2$, and even-order monomial families have order $k \in \{2, 4, \ldots\}$.
For squared families, after an appropriate restriction on the parameter space --- see Assumption~\ref{ass:simple_parameter_space} --- there is a uniquely defined ``square root'', and the squared L2 distance is a proper statistical divergence. 
%A particularly nice property of squared families is that only the parameter and the second moment $\matrix{K}_{\mu, \vector{\psi}} = \int_{\mathbb{X}} \vector{\psi}(\vector{x}) \vector{\psi}(\vector{x})^\top \mu(d\vector{x})$ needs to be known for the normalising constant, divergence, and Fisher information
}
\end{table}

\subsection{Contributions}
%In order to motivate our general setting and results, we first discuss the main special case of the models which we study, called \emph{squared families}.
%We return to this special case in more detail in \S~\ref{sec:squared_fam}.
In \S~\ref{sec:squared_fam}, we focus on squared families, and describe their convenient properties. 
Let $(\mathbb{X},\mathcal{F},\mu)$ be a measure space, with a set $\mathbb{X}$, a $\sigma$-algebra $\mathcal F$ and a reference $\sigma$-finite measure $\mu$, and consider probability density functions with respect to $\mu$ of the form
\begin{align*}
    p(\vector{x} \mid \vector{\theta}) &= \frac{\big(\vector{\theta}^\top \vector{\psi}(\vector{x}) \big)^2}{z(\vector{\theta})}, \numberthis \label{eq:special_case}
\end{align*}
where $\psi:\mathbb{X} \to \mathbb{R}^n$, $z(\vector{\theta}) = \int_{\mathbb{X}} \big(\vector{\theta}^\top \vector{\psi}(\vector{x}) \big)^2 \, \mu(d\vector{x})$, and $\vector{\theta} \in \mathbbold{\Theta} \subseteq \{ \vector{\theta} \in \mathbb{R}^n \mid 0< z(\vector{\theta}) < \infty \}$.
A family of the form $\{ p(\vector{x} \mid \vector{\theta} ) \}_{\vector{\theta} \in \mathbbold{\Theta}}$ is called a \emph{squared family}. 
We study squared families in a ``forward'' fashion, starting with their definition and deriving their properties. 
We discuss their special singularity, derive their Fisher information, introduce a simple technique to remove their singularity, show a tractable factorisation of their normalising constant, and introduce a statistical divergence for the family. 
The Fisher information, normalising constant and statistical divergence are all computationally linked by a single integral called the \emph{squared family kernel}, which only needs to be computed once for the entire family. 
The normalising constant acts as a strictly convex generator, which yields the Fisher information and the statistical divergence.

In \S~\ref{sec:g_fam}, we make steps towards characterising what makes squared families special in the space of broader families.
We derive squared families in a ``backwards'' fashion, starting with key desiderata and proving characterisations leading to squared families. 
We consider families of densities of the form $p(\vector{x} \mid \vector{\theta}) \propto g\big( \vector{\theta}^\top \vector{\psi}(\vector{x}) \big)$, where $g$ is some sufficiently regular nonnegative function. 
We show that $g$-families with a special singularity are characterised as positively homogeneous families, to which even-order monomial and squared families belong.
We derive their Fisher information and show a simple technique which removes their singularity, and show a tractable factorisation of their normalising constant in the special case of even-order monomial families.
After removing the singularity, only for exponential families and positively homogeneous families is the Fisher information a conformal transformation of the Hessian metric derived from a Bregman generator which depends on the parameter only through the normalising constant. 
This is a very computationally convenient property, as it allows for reusing integrals for Fisher information, normalising constants and statistical divergences. 
Beyond exponential families, even-order monomial families additionally possess a tractable factorisation of their normalising constant.
A summary of comparisons between various $g$-families is given in Table~\ref{tab:contrast}.

In \S~\ref{sec:estimation}, we study maximum likelihood estimation in squared families. 
When estimating parameters of well-specified models, standard results on asymptotic normality apply, with precision determined by the Fisher information, which essentially reduces to the squared family kernel.
Under model misspecification, when estimating an arbitrary but well-behaved target density, we show that the KL divergence from the target density to the squared family fit through maximum likelihood estimation is close to the KL divergence from the target density to the best (in the sense of KL divergence) density in the squared family. 
In particular, the difference between these two KL divergences converges in probability to zero at a rate $N^{-1/2}$, where $N$ is the number of datapoints.
Finally, the KL divergence from the target density to the best density in the squared family is bounded by $\mathcal{O}(n^{-1/4})$, using universal approximation results on neural networks, where $n$ is the number of parameters.  

We give all proofs in the Appendix.

\paragraph{Related works}
While models resembling squared families already exist in the literature, to the best of our knowledge (as discussed in \S~\ref{sec:squared_fam}), with a particular emphasis on computational properties and representation ability with respect to other model classes, no prior work has analysed the geometric properties and downstream statistical estimation properties of such models, nor have they placed them within the broad existing framework of universal approximation results. We discuss related constructs.

\section{Background}
\label{sec:roadmap}

%\subsection{Preliminaries and background}
\paragraph{Notation}
We assume that all distributions $P$ and $Q$ admit a probability density function $p$ and $q$ with respect to a common base measure $\mu$. 
Two-argument divergences in parameter, distribution and function space (like $d(\vector{a}:\vector{b})$) are always written with a colon argument separator $:$, even if they may be symmetric.
A summary of notational conventions and important reused symbols is given in Table~\ref{tab:notation}.
\begin{table}[ht!]
    \centering
    \resizebox{\textwidth}{!}{\begin{tabular}{c|c|c} \hline
         Notation & Description & Example  \\ \hline
         Boldface, lowercase & Vector & $\vector{x}$ \\
         Boldface, uppercase & Matrix & $\matrix{X}$ \\
         Blackboard bold & Set & $\mathbb{X}$ \\
         Caligraphic font & Set & $\mathcal{F}$ \\
         Sans serif & Random element & $\vector{\rv{x}}$ \\  \hline \hline
         Name & Symbol & Note  \\ \hline
         Feature size & $n$ & $n \in \{1, 2, \ldots \}$ \\
         & $m$ & $m \in \{1, 2, \ldots \}$ \\
         Dataset size & $N$ & $N \in \{1, 2, \ldots \}$ \\
         Parameter set & $\mathbbold{\Theta}$ & $\mathbbold{\Theta} \subseteq \mathbb{R}^{mn}$ \\
         Canonical parameter & $\matrix{\Theta}$ &  $m \times n$ matrix \\
         Flattened canonical parameter & $\vector{\theta}$ & $\vector{\theta} = \text{vec}(\matrix{\Theta}) \in \mathbbold{\Theta}$ \\
         PSD parameter & $\matrix{M}$ & $\matrix{M} = \matrix{\Theta}^\top \matrix{\Theta} \succeq 0$ \\
         Base support & $\mathbb{X}$ & $\mathbb{X} \subseteq \mathbb{R}^d$, does not depend on $\vector{\theta}$ \\
         Feature / sufficient statistic / hidden layer & $\vector{\psi}$ &  $\vector{\psi}:\mathbb{X} \to \mathbb{R}^{n}$ \\
         Base measure & $\mu$ & \\
         Squared family kernel & $\matrix{K}_{\mu, \vector{\psi}}$ & $\int_{\mathbb{X}} \vector{\psi}(\vector{x})\vector{\psi}(\vector{x})^\top \, \mu(d\vector{x})$ \\ \hline
    \end{tabular}}
    \caption{Notational conventions (top) and some important frequently manipulated objects (bottom) used throughout this paper.}
    \label{tab:notation}
\end{table}
%Let $(\Omega, \mathcal{F}, \mu)$ be a measure space.
%Let $f$ be a function with domain $\mathbb{X} \subseteq \Omega$.
%We will adopt the notation $\mathbb{E}_p[f(\vector{\rv{x}})]$ to mean $\int_{\mathbb{X}} p(\vector{x})  f(\vector{x}) \, \mu (d\vector{x})$ for a (not necessarily nonnegative or normalised) function $p$, and $\sum_{\vector{x} \in \mathbb{X}} p(\vector{x}) f(\vector{x})$ for a function $p$ with a discrete domain $\mathbb{X}$.
%We use $z(p)$ to denote $\mathbb{E}_p[1]$, the \emph{normalising constant}, where finite.
%If $p$ is a \emph{probability density function} with respect to $\mu$, we have $z(p) =1$ and $0 < p(\vector{x}) < \infty$ for $\vector{x} \in \mathbb{X}$.
%We will always call such functions probability density (not mass) functions, even if they are absolutely continuous with respect to counting measure.
%As a result, all probability density functions $p$ may be viewed as Radon Nikodym derivatives of a corresponding probability measure $P$ satisfying $P(d\vector{x}) = p(\vector{x}) \mu (d\vector{x})$, which is absolutely continuous with respect to base measure $\mu$.
%We will use the sign function $\text{sgn}:\mathbb{R} \to \{-1, 1\}$, which maps negative values to $-1$ and nonnegative values to $1$. 

We do not attempt to give a complete background on Fisher information, exponential families, statistical divergences, parameter-space divergences or function-space divergences.
Rather, we assume a certain background and list only the important quantities and qualities we require, referring readers to texts for further background where required.

\subsection{Fisher information}

%We are interested in estimation, so the score function $\nabla_\theta$ is an important object. 
%We consider non-regular models, where the score function to be orthogonal to the parameter.
%Score functions live in an $n-1$ dimensional manifold.  If we want this, we have one zero eigenvalue in the Fisher information
For a probability density $p(\cdot \mid \vector{\theta})$ belonging to some family $\{ p(\cdot \mid \vector{\theta}) \}_{\vector{\theta} \in \mathbbold{\Theta}}$, with $\mathbbold{\Theta}\subseteq \mathbb{R}^n$, under appropriate regularity conditions, the Fisher information $\matrix{G}(\vector{\theta})$ is the expected outer product of the gradient of the log density, 
%and also an integral involving the principal square root of the density,
\begin{align*}
    \matrix{G}_{p}(\vector{\theta}) &= \mathbb{E}_{p(\cdot \mid \vector{\theta})}\Big[ \frac{\partial}{\partial \vector{\theta} }\log p(\vector{\rv{x}}\mid \vector{\theta}) \frac{\partial}{\partial \vector{\theta} }\log p(\vector{\rv{x}}\mid \vector{\theta})^\top \Big]. \label{eq:FIM} \numberthis
    %= 4 \int \frac{\partial}{\partial \vector{\theta} }\sqrt{ p(\vector{x}\mid \vector{\theta})}  \frac{\partial}{\partial \vector{\theta} }\sqrt{ p(\vector{x}\mid \vector{\theta})}^\top \, \mu (d\vector{x}). 
\end{align*}
%the latter being also standard but slightly less common than the former.
Singularities in the Fisher information can present analytical challenges in analysing statistical properties of estimators as well as instantiating estimation procedures, which often rely on a full-rank Fisher information.
Thus, it is usually desirable to avoid zero eigenvalues.
A family $\{ p(\cdot \mid \vector{\theta}) \}_{\vector{\theta} \in \mathbbold{\Theta}}$ is called \emph{regular} if for all $\theta \in \mathbbold{\Theta}$, $\rank(\matrix{G}\big(\vector{\theta})\big)= n $, and \emph{singular} if $\rank(\matrix{G}\big(\vector{\theta})\big)< n$. 
 The Fisher information of the joint distribution of two independent random vectors is the sum of the two Fisher informations of the marginal distributions of the random vectors,
\begin{align*}
p(\vector{x}_1, \vector{x}_2\mid \vector{\theta}) = p_1(\vector{x}_1\mid \vector{\theta})p_2(\vector{x}_2\mid \vector{\theta}) \implies \matrix{G}_{p_1(\cdot) p_2(\cdot)}(\vector{\theta}) = \matrix{G}_{p_1}(\vector{\theta}) + \matrix{G}_{p_2}(\vector{\theta}). \numberthis \label{eq:additive_fisher}
\end{align*}
We assume throughout this paper that the Fisher information is finite and is not equal to $\matrix{0}$.

\subsection{Parameter, statistical and function divergences}
\paragraph{Bregman divergence}
Let $\phi : \mathbbold{\Theta} \to \mathbb{R}$ be a convex and differentiable function on a convex set $\mathbbold{\Theta}$. 
The \emph{Bregman divergence} $d_\phi(\vector{\theta}: \vector{\theta}')$ generated by $\phi$ between two points $\vector{\theta}, \vector{\theta}' \in \mathbbold{\Theta}$ is defined as~\citep[\S 1.3, for example]{amari2016information-geometry}
\begin{align*}
    d_\phi(\vector{\theta}: \vector{\theta}') &= \phi(\vector{\theta}) - \phi(\vector{\theta}') -  \nabla \phi(\vector{\theta}')^\top( \vector{\theta} - \vector{\theta}' ).
\end{align*}
The Bregman divergence is nonnegative, $d_\phi(\vector{\theta}: \vector{\theta}') \geq 0$, for all $\vector{\theta}, \vector{\theta}' \in \mathbbold{\Theta}$, and if $\phi$ is strictly convex, $d_\phi(\vector{\theta}: \vector{\theta}') =0 $ if and only if $\vector{\theta} = \vector{\theta}'$.
%For example, if $\mathbbold{\Theta}$ is the probability simplex, then the strictly convex function defined by $\phi(\vector{\theta})= \vector{\theta}^\top \log\vector{\theta}$ generates the Kullback-Leibler (KL) divergence between discrete distributions.
%General Bregman divergence. KL divergence is tlogt or -logt
%Bregman divergences on function spaces may be defined with an appropriate notion of a derivative, but to avoid such technical definitions, we may instead view $\vector{r}:\mathbb{X} \to \mathbbold{\Theta}$ and $\vector{r}':\mathbb{X} \to \mathbbold{\Theta}$ as functions, and $$D_F(\vector{r}, \vector{r}') = \mathbb{E}_1 \big[ d_F\big(\vector{r}(\vector{\rv{x}}), \vector{r}'(\vector{\rv{x}})\big) \big]$$ as a notion of divergence.
%The allows one to use the generator $F_1(r) = r \log r$ to obtain the KL divergence for continuous probability distributions.

\paragraph{Statistical divergence}
A \emph{statistical divergence} $\widetilde{D}$ on a manifold $\mathbb{M}$ of probability distributions is a nonnegative twice continuously differentiable function $\widetilde{D}:\mathbb{M} \times\mathbb{M} \to [0,\infty)$ satisfying two additional axioms~\citep[\S 1.2, for example]{amari2016information-geometry}. 
Firstly, $\widetilde{D}(p,p') = 0$ if and only if $p=p'$.
Secondly, for infinitesimal displacements $dp$, the divergence $\widetilde{D}(p, p+dp)$ is a positive definite quadratic form.

\paragraph{$f$-divergence}
A class of statistical divergences is given by the \emph{$f$-divergences}.
Let $f$ be a convex function and let $p$ and $p'$ be two probability density functions with respect to base measure $\mu$. 
The $f$-divergence of $p$ from $p'$ is defined as~\citep[\S 3.10, for example]{nielsen2020elementary}
\begin{align*}
    D_f(p:p') &= \mathbb{E}_{p'}\Big[ f\Big( \frac{p(\vector{\rv{x}})}{p'(\vector{\rv{x}})}  \Big) \Big].%\int_{\mathbb{X}} f\big( \frac{p(\vector{x})}{p'(\vector{x})}  \big) p'(\vector{x}) \, \mu (d\vector{x}).
\end{align*}

\paragraph{Fisher information metric and $f$-divergences}
Every $f$-divergence induces the Fisher information as the positive definite form appearing in the statistical divergence between infinitesimally displaced elements~\citep[\S 3.5, for example]{amari2016information-geometry}.
The Fisher information describes a lower-bound on the variance of any estimator through the Cramer-Rao bound, and is also useful for describing the limits of estimation procedures.
For example, informally, under mild regularity conditions, the maximum likelihood estimate of a well-specified statistical model asymptotically approaches a Gaussian distribution with mean equal to the true parameter and variance equal to the inverse Fisher information divided by the number of samples.
More generally, maximum likelihood estimates of misspecified models can also be analysed using a quantity similar to the Fisher information~\citep{white1982maximum}.

\paragraph{Hessian metric}
By the mean-value form of Taylor's theorem, the Bregman divergence generated by $\phi$ can be expressed as the integral of the Hessian of $\phi$ along the line connecting the arguments of the Bregman divergence. 
Thus, the Hessian $\nabla^2 \phi:\mathbbold{\Theta} \to \mathbb{S}_+^{n}$ of the Bregman generator $\phi$ gives a metric on the parameter manifold that we call the \emph{Hessian metric}.

\paragraph{$f$-divergence inequalities}

Three examples of $f$-divergences, which we use in our analysis, are the Kullback-Leibler (KL) divergence, the squared Hellinger (SH) distance and the total variation (TV) distance.
%The KL divergence is obtained when $f(z) = z \log z$.
%The KL divergence is the only divergence that is both an $f$-divergence and a Bregman divergence.
The KL divergence is
\begin{align*}
    \kl[p:p'] = D_{(\cdot) \log(\cdot)} [p:p'] &= \mathbb{E}_p \log\frac{p(\vector{\rv{x}})}{p'(\vector{\rv{x}})}.  \numberthis \label{eq:KL}
\end{align*}
The \emph{squared L2 distance} between any two (potentially vector-valued) square integrable functions $\vector{a}$ and $\vector{b}$ is denoted
\begin{align*}
    \sqltwo(\vector{a}: \vector{b}) &= \frac{1}{2} \int_{\mathbb{X}} \big\Vert\vector{a}(\vector{x}) - \vector{b}(\vector{x})\big\Vert_2^2 \, \mu (d\vector{x}).
\end{align*}
This is a useful divergence in quantifying approximations of target functions using machine learning models, for example neural networks with random hidden layers.
For example, this divergence quantifies the approximation of random Fourier feature models~\citep{rahimi2008weighted}.
The SH distance on probability densities is related to the $\sqltwo$ distance on square root densities, and is
\begin{align*}
    \sqhel[p:p'] =  D_{1 - \sqrt{\cdot}} [p:p'] &=  \sqltwo \big(\sqrt{p(\cdot)}: \sqrt{q(\cdot)} \big). \numberthis \label{eq:sqHellinger}
\end{align*}
The TV distance is
\begin{align*}
    \tv[p:p'] = D_{\frac{1}{2}|\cdot - 1|}[p:p'] &= \frac{1}{2}\int_{\mathbb{X}} | p(\vector{x}) - p'(\vector{x})| \, \mu (d\vector{x}).
\end{align*}

\subsection{Exponential families}
The log normalising constant $\phi(\vector{\theta}) = \log \int_{\mathbb{X}} \exp\big( \vector{\theta}^\top \vector{\psi}(\vector{x}) \big)\, \mu(d\vector{x})$ is used as Bregman generator for a regular exponential family. 
In this case, the Bregman divergence $d_\phi$ is equal to the reverse KL divergence, which is a statistical divergence.
The Hessian metric in parameter space induced by $\phi$ is the Fisher information metric over the manifold of probability distributions. 
\section{Squared families}
\label{sec:squared_fam}
We introduce squared families and discuss some of their convenient properties in this section. 
These are shown in a ``forward'' fashion, starting with the definition of squared families and finding their properties.
Later in \S~\ref{sec:g_fam} the properties are used in a ``backwards'' fashion, starting with desiderata and showing characterisations leading to the squared family.
%We now return to the special case of \S~\ref{sec:roadmap}, and explicate~\ref{eq:d1} and~\ref{eq:d2}. We further investigate their statistical properties, including but not limited to~\ref{eq:d3_special}.
%\subsection{Squared family densities}
%Squared families are a special case of positively homogeneous and even-order monomial families, which in turn are special cases of $g$-families, as characterised in \S~\ref{sec:g_fam}. 
%For clarity, we first discuss the general case here.
%In analogy with exponential family models, we now define squared family densitys, by taking $\omega$ to be linear in a finite-dimensional statistic $\vector{\psi}$ of $\vector{\rv{x}}$.
\begin{definition}[Squared family]
    \label{def:squared_prob_model}
Let $(\mathbb{X}, \mathcal{F}, \mu)$ be a measure space, with $\mathbb{X} \subseteq \mathbb{R}^d$. 
Let $\vector{\psi}:\mathbb{X} \to \mathbb{R}^n$, and $\vector{\theta} \in \mathbbold{\Theta} \subseteq \mathbb{R}^n$.
%Choosing a model linear in $\vector{\psi}$ for $\omega$ in~\eqref{eq:squared_prob_set}, 
We say that 
\begin{align*}
    p(\vector{x}\mid \vector{\theta}) &= \frac{\big(  \vector{\theta}^\top \vector{\psi}(\vector{x}) \big)^2}{ z(\vector{\theta})}, \qquad \text{where} \qquad z(\vector{\theta}) =  \int_{\mathbb{X}} \big(  \vector{\theta}^\top \vector{\psi}(\vector{x}') \big)^2 \, \mu (d\vector{x}') \numberthis \label{eq:squared_prob_orig}
\end{align*}
belongs to a squared family. 
We call the collection $\{ p(\vector{x}\mid \vector{\theta}) \}_{\vector{\theta} \in \mathbbold{\Theta}}$ a squared family, where $\mathbbold{\Theta} \subseteq \{ \vector{\theta} \in \mathbb{R}^n \mid 0 < z(\vector{\theta}) < \infty \}$.
\end{definition}
%The nomenclature \textit{squared family} comes from the fact that $p(\vector{x}\mid \vector{\theta}) \propto \big( \vector{\theta}^\top \vector{\psi}(\vector{x}) \big)^2$, i.e. a squared family is an instance of~\eqref{eq:transformed_linear} with $g(\cdot) = (\cdot)^2$.

\subsection{Orthogonal singularity and Fisher information}
\label{sec:squared_singular}
A conspicuous property of $m$-squared families is their singular and non-identifiable nature.
For any $\vector{\theta} \in \mathbbold{\Theta}$ and $ s \neq 0$, densities with parameters $\vector{\theta}$ and $s \vector{\theta}$ are the same, that is, $p(\vector{x} \mid \vector{\theta}) = p(\vector{x} \mid s \vector{\theta})$. 
This conveniently manifests in that the score function is always orthogonal to the parameter,
\begin{align*}
    \vector{\theta}^\top \nabla_{\vector{\theta}} \log p(\vector{x} \mid \vector{\theta}) = 0.
\end{align*}
This orthogonal singularity implies that the Fisher information, which is the expected outer product of the score function, has an eigenvector of $\vector{\theta}$ with eigenvalue $0$.
Indeed, by direct calculation, or as a special case of our later results in \S~\ref{sec:g_fam}, the Fisher information is
\begin{align*}
    \matrix{G}_p(\vector{\theta}) = 4\matrix{K}_{\mu, \vector{\psi}}/z(\vector{\theta}) - 4 \matrix{K}_{\mu, \vector{\psi}} \vector{\theta} \vector{\theta}^\top \matrix{K}_{\mu, \vector{\psi}}/z(\vector{\theta})^2,
\end{align*}
where 
\begin{align*}
    \matrix{K}_{\mu, \vector{\psi}} = \int_{\mathbb{X}} \vector{\psi}(\vector{x})\vector{\psi}(\vector{x})^\top \, \mu(d\vector{x})
\end{align*}
is called the \emph{squared family kernel}, which is independent of parameter $\vector{\theta}$.
We assume that the squared family kernel is finite. 
This expression for the Fisher information  reveals that it is the difference between a positive semidefinite matrix and a rank-$1$ matrix proportional to $\matrix{K}_{\mu, \vector{\psi}} \vector{\theta} (\matrix{K}_{\mu, \vector{\psi}} \vector{\theta})^\top$, and hints at a means of correcting the singularity.
Recalling~\eqref{eq:additive_fisher} that the Fisher information of the joint distribution of two independent random elements is the sum of their individual Fisher information, and the Fisher information of a univariate Gaussian $\rv{a} \sim \mathcal{N}\big(\log z(\vector{\theta}), 1\big) = p_0(a\mid\theta)$ is $4 \matrix{K}_{\mu, \vector{\psi}} \vector{\theta} \vector{\theta}^\top \matrix{K}_{\mu, \vector{\psi}}/z(\vector{\theta})^2$, the Fisher information of $(\vector{\rv{x}}, \rv{a}) \sim p(\vector{x} \mid \vector{\theta}) p_0(a \mid \vector{\theta})$ is
\begin{align*}
    \matrix{G}_{p(\cdot) p_0(\cdot)}(\vector{\theta}) = 4\matrix{K}_{\mu, \vector{\psi}}/z(\vector{\theta}).
\end{align*}
We refer to this technique as \emph{dimension-augmentation}.
If the parameter-independent integral $\matrix{K}_{\mu, \vector{\psi}}$ is positive definite, this yields a regular model for which standard statistical estimation procedures can be analysed. 
Furthermore, given any sample $\{ \vector{x}_i \}_{i=1}^N$ from density $p(\vector{x} \mid \vector{\theta})$ belonging to a squared family, a sample $\{ (\vector{x}_i, a_i) \}_{i=1}^N$ belonging to a dimension-augmented squared family $p(\vector{x} \mid \vector{\theta}) p_0(a \mid \vector{\theta})$ may be formed by choosing $z(\vector{\theta})$ to be some value (say $1$) without imposing any constraint or assumption on the marginal $p(\vector{x} \mid \vector{\theta})$.

\begin{remark}
It is curious that the Fisher information of the dimension-augmented squared family is conformally equivalent to a Hessian metric generated by a function $\chi$ depending only on the normalising constant, with conformal transformation $c$,
\begin{align*}
    \matrix{G}(\vector{\theta}) = \nabla^2 \chi\big(z(\vector{\theta})\big) / c\big( z(\vector{\theta}) \big),
\end{align*}
where $\chi$ is the identity function and $c(z) = z/c$.
Another instance of such a conformal equivalence is in exponential families, where $\chi$ is the logarithmic function, and $c$ is constant $1$. 
Exponential families and dimension-augmented squared families are in some sense rare families which have Hessian generators $\chi$ depending on parameter $\vector{\theta}$ only through $z(\vector{\theta})$ which are conformally equivalent to the Fisher information $\matrix{G}(\vector{\theta})$, as we later make more precise in \S~\ref{sec:g_fam}.
\end{remark}

\subsection{Normalising constant}
By expanding the square of the inner product $\vector{\theta}^\top \vector{\psi}(\vector{x})$, the normalising constant $z(\vector{\theta})$ admits a \emph{\textcolor{DarkOrchid}{parameter}-\textcolor{OliveGreen}{integral} factorisation},
\begin{align*}
    z(\vector{\theta}) &= \Tr \Bigg( \textcolor{DarkOrchid}{\vector{\theta} \vector{\theta}^\top} \textcolor{OliveGreen}{\overbrace{\int_{\mathbb{X}} \vector{\psi}(\vector{x}) \vector{\psi}(\vector{x})^\top \, \mu (d\vector{x})}^{\matrix{K}_{\mu, \vector{\psi}}=}} \Bigg)
\end{align*}
The normalising constant is thus an inner product (in $\mathbb{R}^{n^2}$) of a function depending only on the parameter and an integral that depends on $\vector{\psi}$, $\mathbb{X}$ and $\vector{\mu}$ but not $\vector{\theta}$.
This means that the normalising constant $z(\vector{\theta})$ for any member in the family $\{ p(\vector{x} \mid \vector{\theta} ) \}_{\vector{\theta} \in \mathbbold{\Theta}}$ can be computed using the same computation for the squared family kernel $\matrix{K}_{\mu, \vector{\psi}}$. 
Even if the integral is not available in closed form (and it often is, see~\citet[Table 1]{tsuchida2023squared}), \emph{approximations of the squared family kernel $\matrix{K}_{\mu, \vector{\psi}}$ can be reused for the calculation of the normalising constant for every member of the family.}
Furthermore, approximations for the squared family kernel together with the parameter-integral factorisation can be used to easily approximate the Fisher information for every element in the family.

\subsection{Statistical divergence and Bregman divergence}
\label{sec:infogeo}
Having established the connection between the Fisher information metric and normalising constant, which can be approximated via the parameter-integral factorisation, we now turn our attention to statistical divergences which are also linked to normalising constants.
In order to do so, we introduce two parameter spaces %for both $m=1$ and $m=n$ cases, which do not restrict the family of probability measures.
%When $m=1$, 
which are a half space and boundary of an ellipsoid.
\begin{figure}
    \includegraphics{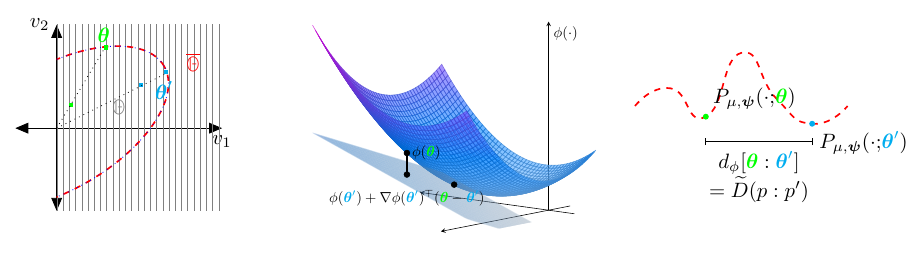}
    \caption{Useful parameter spaces for squared families. (Left) The space $\mathbbold{\Theta}=\{ \vector{\theta} \in \mathbb{R}^n \mid \theta_1 > 0\}$ removes ambiguity in the sign of the parameter $\vector{\theta}$. 
    When using appropriate dimension-augmentation (\S~\ref{sec:dim_aug}), it results in an identifiable model with a full-rank Fisher information (Lemma~\ref{lemma:gaussian_fim}). 
    The boundary of the half ellipsoid $\overline{\mathbbold{\Theta}}= \{ \vector{\theta} \in \mathbbold{\Theta} \mid \vector{\theta}^\top \matrix{K}_{\mu, \vector{\psi}} \vector{\theta} = 1\}$ additionally allows for a proper statistical divergence (Right) by removing ambiguity in the scale of the parameter $\vector{\theta}$ (Left). This statistical divergence is equal to a Bregman divergence (Middle) generated through the normalising constant $z(\vector{\theta}) = \vector{\theta}^\top \matrix{K}_{\mu, \vector{\psi}} \vector{\theta}$ restricted to $\overline{\mathbbold{\Theta}}$. }
\end{figure}
\begin{assumption}
\label{ass:simple_parameter_space}
Define the half-space
    $$\mathbbold{\Theta} = \{ \vector{\theta} \in \mathbb{R}^n \mid \theta_1 > 0\},$$
and the boundary of the half-ellipse
$$ \overline{\mathbbold{\Theta}} = \{ \vector{\theta} \in \mathbbold{\Theta} \mid \vector{\theta}^\top \matrix{K}_{\mu, \vector{\psi}} \vector{\theta} = 1\}.$$
\end{assumption}
The subset $\mathbbold{\Theta} \subset \mathbb{R}^n$ orients the parameter space, thus removing a redundancy in the sign of the parameters.
The subset $\overline{\mathbbold{\Theta}} \subset \mathbbold{\Theta}$ further normalises the parameter space, thus removing a redundancy in the value of the normalising constant.
These parameter spaces impose no restrictions on the family of probability measures. That is, $$\{ p(\cdot \mid \vector{\theta}) \}_{\vector{\theta} \in \mathbb{R}^n \setminus \{ \vector{0}\}} = \{ p(\cdot \mid \vector{\theta}) \}_{\vector{\theta} \in \mathbbold{\Theta}} = \{ p(\cdot \mid \vector{\theta}) \}_{\vector{\theta} \in \overline{\mathbbold{\Theta}}}.$$
We also require some regularity on the features and the moments of the features in order to ensure that we can build a proper statistical divergence.
\begin{assumption}
    \label{ass:strictly_pd}
    $\matrix{K}_{\mu, \vector{\psi}}$ is strictly positive definite.
\end{assumption}
\begin{assumption}
    \label{ass:rich_features}
    The span of $\{ \vector{\psi}(\vector{x}) \mid \vector{x} \in \mathbb{X}\}$ is all of $\mathbb{R}^n$.
\end{assumption}
Under this setup, we can describe the statistical divergence as built from a Bregman divergence, which is essentially a special squared Mahalanobis distance restricted to the ellipse.

\begin{restatable}{theorem}{divergencevec}
\label{thm:divergencevec}
    The Bregman divergence generated by $\phi(\vector{\theta}) = z(\vector{\theta}) = \vector{\theta}^\top \matrix{K}_{\mu, \vector{\psi}} \vector{\theta}$ is twice the squared $L^2$ distance between the functions whose squares are proportional to the probability densities,
    $$d_\phi[\vector{\theta}: \vector{\theta}'] = 2 \sqltwo\big( \vector{\theta}^\top \vector{\psi}(\cdot): \vector{\theta}'^\top \vector{\psi}(\cdot) \big)$$
    Suppose that Assumptions~\ref{ass:simple_parameter_space}, \ref{ass:strictly_pd} and~\ref{ass:rich_features} hold. Then ${\sqltwo\big( \vector{\theta}^\top \vector{\psi}(\cdot): \vector{\theta}'^\top \vector{\psi}(\cdot) \big)}$ is a statistical divergence on $\overline{\mathbbold{\Theta}} \times \overline{\mathbbold{\Theta}}$.
\end{restatable}

It is helpful to recall the analogous results for any exponential family, where $z(\vector{\theta}) = \int_{\mathbb{X}} \exp\big(\vector{\theta}^\top \psi(\vector{x}) \big) \, \mu(d\vector{x})$.
The log-partition function $\chi(z) = \log z$ is used as a convex generator $\phi(\vector{\theta}) = \log z(\vector{\theta})$ of a Bregman divergence.
This Bregman divergence is equal to the KL divergence $\kl$ when evaluated over the exponential family, which is a statistical divergence. 
The dual parameter which is the derivative $\nabla \phi(\vector{\theta})$ of the generator, is the expected value of the sufficient statistic $\mathbb{E}_{p(\cdot \mid \vector{\theta})}[\vector{\psi}(\vector{\rv{x}})]$.
The second derivative of the generator is the covariance of the sufficient statistic, which is equal to the Fisher information.

\subsection{Other considerations and related work}
\paragraph{Matrix-valued parameters}
We may generalise squared families so that they are proportional to squared norms of vector-valued linearly parameterised functions. 
In other words, with matrix-valued parameter $\matrix{\Theta} \in \mathbb{R}^{m \times n}$ and flattened parameter $\vector{\theta} = \text{vec}(\matrix{\Theta})$, we may consider the model $p(\vector{x}\mid \vector{\theta}) \propto \Vert \matrix{\Theta} \vector{\psi}(\vector{x}) \Vert_2^2 =  \Tr \big(\matrix{\Theta} \vector{\psi}(\vector{x}) \vector{\psi}(\vector{x})^\top \matrix{\Theta}^\top \big)$.
%Using the cyclic property of the trace and the Kronecker product identity $\text{vec}(\matrix{A} \matrix{B} \matrix{C}) = (\matrix{C}^\top \otimes \matrix{A}) \text{vec}(\matrix{B})$, we obtain $p(\vector{x}\mid \vector{\theta}) \propto \vector{\theta}^\top  \matrix{K}_{\vector{\psi}}^{(m)}(\vector{x}) \vector{\theta}$.
This generalisation is not always as analytically convenient as Definition~\ref{def:squared_prob_model}, but still retains some of its easy to analyse properties.
\begin{definition}[$m$-squared family]
\label{def:m-squared_prob_model}
Let $(\mathbb{X}, \mathcal{F}, \mu)$ be a measure space, with $\mathbb{X} \subseteq \mathbb{R}^d$. 
Let $\vector{\psi}:\mathbb{X} \to \mathbb{R}^n$, and $\vector{\theta} \in \mathbbold{\Theta} \subseteq \mathbb{R}^{mn}$.
%Choosing a model linear in $\vector{\psi}$ for $\omega$ in~\eqref{eq:squared_prob_set}, 
We say that 
\begin{align*}
    p^{(m)}(\vector{x}\mid \vector{\theta}) &= \frac{ \Vert \matrix{\Theta} \vector{\psi}(\vector{x}) \Vert_2^2 }{ z(\vector{\theta}) }, \qquad \text{where} \qquad z(\vector{\theta}) = \int_{\mathbb{X}} \Vert \matrix{\Theta} \vector{\psi}(\vector{x}') \Vert_2^2  \, \mu(d\vector{x}') \numberthis \label{eq:squared_prob_orig}
\end{align*}
%where $\matrix{K}_{\vector{\psi}}^{(m)}(\vector{x}) = \vector{\psi}(\vector{x}) \vector{\psi}(\vector{x})^\top   \otimes \matrix{I}_{m \times m}$, 
belongs to an $m$-squared family. 
We call the collection $\{ p^{(m)}(\vector{x}\mid \vector{\theta}) \}_{\vector{\theta} \in \mathbbold{\Theta}}$ an $m$-squared family,  where $\mathbbold{\Theta} \subseteq \{ \vector{\theta} \in \mathbb{R}^{mn} \mid 0 < z(\vector{\theta}) < \infty \}$.
\end{definition}
Thus a squared family density is proportional to a squared norm of $n$-dimensional vector according to the rank-$1$ quadratic form $\matrix{K}_{\vector{\psi}}(\vector{x}) = \vector{\psi}(\vector{x}) \vector{\psi}(\vector{x})^\top$, whereas an $m$-squared family density is proportional to a squared norm of an $nm$-dimensional vector according to the rank-$m$ quadratic form $\matrix{K}_{\vector{\psi}}^{(m)}(\vector{x}) = \vector{\psi}(\vector{x}) \vector{\psi}(\vector{x})^\top \otimes \matrix{I}_{m \times m}$.
The rank-$1$ model can in some instances simplify calculations, but certain simplifications can be made even with general $m > 1$.
For example, the normalising constant $z(\vector{\theta})$ is given by
\begin{align*}
    z(\vector{\theta}) &= \Tr \Bigg(\vector{\theta} \vector{\theta}^\top \underbrace{\Big( \overbrace{\int_{\mathbb{X}} \vector{\psi}(\vector{x}) \vector{\psi}(\vector{x})^\top  \, \mu (d\vector{x})}^{\matrix{K}_{\mu, \vector{\psi}}=}  \otimes \matrix{I}_{m \times m} \Big)}_{=: \matrix{K}_{\mu, \vector{\psi}}^{(m)}} \Bigg),
\end{align*}
and therefore still satisfies a parameter-integral factorisation.
When $m=n$, the appropriate generalisation of the restriction of the parameter space is in terms of the Cholesky decomposition or similar factorisation. 
\begin{assumption}
\label{ass:simple_parameter_space_m}
Let $m = n$. 
Denote the space of $n \times n$ lower triangular matrices with strictly positive diagonals (i.e. the space of Cholesky factors) by $\mathbb{L}$.
Define the flattened space of Cholesky factors
    $$\mathbbold{\Theta} = \{ \vector{\theta} = \text{vec}(\matrix{L}) \mid \matrix{L} \in \mathbb{L} \},$$
and the boundary of the half-ellipse
$$ \overline{\mathbbold{\Theta}} = \{ \vector{\theta} \in \mathbbold{\Theta} \mid \vector{\theta}^\top \matrix{K}_{\mu, \vector{\psi}}^{(m)} \vector{\theta} = 1\}.$$
\end{assumption}
Under this restricted parameter space, we also have a statistical divergence.
\begin{restatable}{theorem}{divergencemat}
\label{thm:divergencemat}
    The Bregman divergence generated by $\phi(\vector{\theta}) = z(\vector{\theta}) = \vector{\theta}^\top \matrix{K}_{\mu, \vector{\psi}}^{(m)} \vector{\theta}$ is twice the squared $L^2$ distance between the functions whose norms are proportional to the probability densities,
    $$d_\phi[\vector{\theta}: \vector{\theta}'] = 2 \sqltwo\big( \matrix{\Theta} \vector{\psi}(\cdot): \matrix{\Theta}' \vector{\psi}(\cdot) \big)$$
    Suppose that Assumption~\ref{ass:simple_parameter_space} or Assumption~\ref{ass:simple_parameter_space_m} holds, and Assumption~\ref{ass:strictly_pd} and Assumption~\ref{ass:rich_features} holds. Then $\sqltwo\big( \matrix{\Theta} \vector{\psi}(\cdot): \matrix{\Theta}' \vector{\psi}(\cdot) \big)$ is a statistical divergence on $\overline{\mathbbold{\Theta}} \times \overline{\mathbbold{\Theta}}$.
\end{restatable}

\paragraph{Alternate parameterisations and unidentifiability}
Noting the cyclic property of the trace and the linearity of the integral, we may also write the density in an inner product form,
\begin{align*}
    p^{(m)}(\vector{x}\mid \vector{\theta}) &= \frac{  \Tr \Big( \matrix{\Theta}^\top \matrix{\Theta} \vector{\psi}(\vector{x}) \vector{\psi}(\vector{x})^\top \Big)}{  \Tr \Big( \matrix{\Theta}^\top \matrix{\Theta} \matrix{K}_{\mu, \vector{\psi}} \Big)}, \numberthis \label{eq:squared_prob}
\end{align*}
where $\matrix{K}_{\mu, \vector{\psi}} = \int_{\mathbb{X}} \vector{\psi}(\vector{x}) \vector{\psi}(\vector{x})^\top \, \mu (d\vector{x})$.
Since~\eqref{eq:squared_prob} depends on $\matrix{\Theta}^\top \matrix{\Theta}$, we may alternatively parameterise the squared family in terms of a PSD matrix $\matrix{M} = \matrix{\Theta}^\top \matrix{\Theta}$.
The $\matrix{M}$-parameterisation also highlights the singular nature of $m$-squared families.
The $\matrix{M}$-parameterisation uniquely defines $p^{(m)}(\vector{x}\mid \vector{\theta})$, but does not uniquely define $\vector{\theta}$.
There are multiple $\vector{\theta}$ which result in the same $p^{(m)}(\vector{x}\mid \vector{\theta})$.
There are two \emph{distinct} factors contributing to unidentifiability: (1) the decomposition of a PSD matrix $\matrix{M} = \matrix{\Theta}^\top \matrix{\Theta}$ is not unique, and (2) the matrix $\matrix{M}$ appears in both a numerator and denominator, and therefore $p^{(m)}(\vector{x}\mid \vector{\theta})$ is invariant to scales in its parameter $\vector{\theta}$.

\paragraph{Mixture models}
Inspecting~\eqref{eq:squared_prob} and restricting $\matrix{\Theta}^\top \matrix{\Theta}$ to be diagonal, we see that $m$-squared families generalise mixture models with fixed component densities $\vector{\psi}(\vector{x}) \odot \vector{\psi}(\vector{x}) $.
Such models can be reparameterised by projecting parameters onto the simplex and properly normalising the component densities.
Under such reparameterisation, the information geometry of mixture models is known~\citep[\S~2.3, for example]{amari2016information-geometry}. 
The negative entropy is the dual of a Bregman generator. 
This dual Bregman divergence is the KL divergence. 
Note that such mixture families consider \emph{fixed} component densities, unlike some other settings such as Gaussian mixture models with parameters including the mean and covariance of the component densities~\citep[for example]{mclachlan2019finite}.

\paragraph{Kernel methods}
Instances of similar models include non-parametric reproducing kernel Hilbert space (RKHS) models for nonnegative functions, where the Euclidean inner product $\Tr\big(\matrix{M} \vector{\psi}(\vector{x}) \vector{\psi}(\vector{x})^\top \big) = \langle \matrix{M} \vector{\psi}(\vector{x})  , \vector{\psi}(\vector{x}) \rangle$ is replaced by an inner product in an RKHS $\mathbb{H}$, and $\vector{\psi}$ is taken to be the feature mapping of an RKHS~\citep{marteau2020non,rudi_ciliberto,marteau2022sampling}. 
The numerator $\langle \matrix{M}\vector{\psi}(\vector{x}), \vector{\psi}(\vector{x}) \rangle_{\mathbb{H}}$ of such models may be viewed as the squared norm of an element of RKHS, and is linear in $\matrix{M}$ but quadratic in $\matrix{\Theta}$.
When the kernel is universal, such models are universal approximators~\citep[Theorem 3 and Theorem 4]{marteau2020non}, and when elements of such a class minimise a regularised functional, they admit a representation in terms of linear combinations of finitely many evaluations of the canonical feature mapping~\citep[Theorem 3 and Theorem 4, Theorem 1]{marteau2020non}, mirroring standard properties of linear kernel methods~\citep{micchelli2006universal}~\citep[Theorem 1]{scholkopf2001generalized}.
Normalising constants $\Tr\big( \matrix{M} \matrix{K}_{\mu, \vector{\psi}} \big)$ are explicitly computed in closed-form in the case where $\vector{\psi}$ is a Gaussian function.
~\citet{rudi_ciliberto} study such models, also with Gaussian features, and show that they are sample-efficient in the sense that they obtain the optimal sample complexity for density estimation of $\beta$-times differentiable densities, by minimising a regularised empirical estimate of the squared L2 distance between the model and observed data.
Finally, also in the case of Gaussian features,~\citet{marteau2022sampling} derive a method to sample from such probability models.

\paragraph{Squared neural families} When $\vector{\psi}(\vector{x}) = \sigma\big(\matrix{W} \vector{t}(\vector{x}) + \vector{b}\big)$ is the evaluation of the hidden layer of a neural network with activation function $\sigma$, weight matrix $\matrix{W} \in \mathbb{R}^{n \times D}$, bias vector $\vector{b} \in \mathbb{R}^n$, and sufficient statistic\footnote{This sufficient statistic $\vector{t}$ is different to the sufficient statistic $\vector{\psi}$; both are accurately called sufficient statistics, although $\vector{t}$ is also a sufficient statistic for $\matrix{W}, \vector{b}$.} $\vector{t}:\mathbb{X} \to \mathbb{R}^D$, then $p$ is said to belong to a \emph{squared neural family} (SNEFY), indexed by parameters $\matrix{\Theta}, \matrix{W}$ and $\vector{b}$~\citep{tsuchida2023squared}. 
Again, this terminology arises because such a density model amounts to normalising the evaluation $\Vert \matrix{\Theta} \sigma\big(\matrix{W} \vector{t}(\matrix{x}) + \vector{b}\big) \Vert_2^2$ of the squared norm of a neural network. 
Remarkably, such density models admit closed-form expressions for $\matrix{K}_{\mu, \vector{\psi}}$ (and therefore, the normalising constant) for a wide range of activation functions $\sigma$, base measures $\mu$ and supports $\mathbb{X}$~\citep[Table 1]{tsuchida2023squared}, beyond the previously explicated Gaussian features $\vector{\psi}$.
This is due to the vast number of similar integrals that have been computed in closed form in the context of infinite width neural networks and Gaussian processes~\citep{neal1995bayesian}.
Such families are also closed under conditioning, and admit closed-form marginal distributions.
SNEFY models allow for adaptive features $\vector{\psi}$ through learnable parameters $\matrix{W}$ and $\vector{b}$.
Beyond density modelling, such models have also been applied to the related problem of Poisson point process intensity modelling~\citep{tsuchida2024snefy-ppp}.
A single draw (consisting of $e$ events) from a Poisson point process may be thought of as $e$ i.i.d. samples from the density obtained by normalising the intensity function, where $e$ is itself a random variable following a Poisson distribution with a rate parameter given by an integral of the intensity function~\citep[\S~8.5]{Cressie1993}~\citep[\S~4.3]{Baddeley2007}. 
One notable result is that under squared family densities, intensity modelling results in a negative log likelihood (NLL) which is convex in parameter $\matrix{M}$, whereas density modelling results in an NLL which is nonconvex in parameter $\matrix{M}$.
%\textcolor{red}{We will find parameterisations of squared family densitys that result in convex NLLs, \ldots}

\paragraph{Squared Gaussian processes}
Also in the context of intensity estimation of Poisson point processes, several works use squared norms of Gaussian processes to model intensity functions~\citep{mccullagh2006permanental,lloyd2015variational,walder2017fast,kim2022fast,sellier2023sparse}.
Modelling intensity functions with squared Gaussian processes replaces the computational or analytical challenge of computing the normalising constant with the computatonal or analytical challenge of computing the integrated intensity function.
Their frequentist counterpart of using a squared element of an RKHS~\citep{flaxman2017poisson} resembles the squared kernel methods for density modelling, where $\matrix{M}$ is restricted to be rank $1$; integrals can be approximated by utilising the notion of an equivalent kernel~\citep[\S~7.1]{williams2006gaussian}.

\paragraph{Probabilistic circuits}
Similar squared probability models appear in the literature of probabilistic circuits~\citep{choiprobabilistic}, where it has also been noted that squaring functions allows for tractable normalisation and composition within probabilistic circuits, where they are known as squared circuits, for which various representation results exist in terms of (compositions of) other probabilistic circuits~\citep{loconte2024sum,loconte2023negative,loconte2024subtractive,wang2024relationship}, and have been applied in turning knowledge graph embedding models into generative models~\citep{loconte2023how}.

\paragraph{Marginal and conditional distributions}
We can compute the marginal and conditional distributions of squared family densities, generalising a property of SNEFY models~\citep[Theorem 1 and 2]{tsuchida2023squared}.

\begin{restatable}{theorem}{marginalsjoint}
\label{thm:marginals_joints}
    Let $\vector{\rv{x}} = (\vector{\rv{x}}_1, \vector{\rv{x}}_2) \sim p^{(m)}(\cdot\mid \vector{\theta})$ jointly follow an $m$-squared family density. 
    Assume that the base measure $\mu(d\vector{x})$ factors as $\mu(d\vector{x}) = \mu_1(d\vector{x}_1 \mid \vector{x}_2) \mu_2(d \vector{x}_2)$.
    Define $\vector{\psi}_{1\mid 2} (\vector{x}_1) = \vector{\psi}\big( (\vector{x}_1, \vector{x}_2) \big)$. 
    Then 
    \begin{itemize}
    \item the conditional distribution of $\vector{\rv{x}}_1$ given $\vector{\rv{x}}_2 = \vector{x}_2$ is a squared family density, $\vector{\rv{x}}_1 \sim P_{\mu_1(\cdot \mid \vector{x}_2), \vector{\psi}_{1\mid 2}}^{(m)}(d \vector{x}_1\mid \vector{\theta}) $, and
    \item the marginal distribution of $\vector{\rv{x}}_2 \sim P_2(\cdot \mid \vector{\theta})$ (which is not necessarily a squared family density), is
    $$
    P_2(d\vector{x}_2\mid\vector{\theta}) = \frac{\Tr\big( \matrix{M} \matrix{K}_{\mu_1(\cdot \mid \vector{x}_2), \vector{\psi}_{1 \mid 2}}\big)}{\Tr\big( \matrix{M} \matrix{K}_{\mu, \vector{\psi}} \big)} \mu_2(d\vector{x}_2).$$
    \end{itemize}
\end{restatable}
The normalising constant of the conditional distribution, $\Tr( \matrix{M} \matrix{K}_{\mu_1(\cdot \mid \vector{x}_2), \vector{\psi}_{1 \mid 2}})$ can often be computed in closed form whenever the normalising constant of the joint distribution can be computed in closed form. 
For example, if $\mu$ is a Gaussian measure, so is $\mu_1(\cdot \mid \vector{x}_2)$, and if $\vector{\psi}$ is the hidden layer of a neural network, so is $\vector{\psi}_{1 \mid 2}$, with the same activation function but with a modified bias parameter. 
As long as the kernel admits a closed form for fixed activation function, all Gaussian measures and all bias parameters, the kernel for the conditional distribution can be computed in closed form.
See~\citet{tsuchida2023squared} for examples.
As one might expect, the expression for the marginal distribution involves the same kernel matrix $\matrix{K}_{\mu_1(\cdot \mid \vector{x}_2), \vector{\psi}_{1 \mid 2}}$ as the conditional distribution.

\iffalse
\paragraph{Estimating expectations and importance sampling}
Let $\vector{\rv{x}} \sim p^{(m)}(\cdot\mid \vector{\theta})$ and let $\vector{h}$ be some suitably well-behaved measurable test function.
Suppose $\mu(\mathbb{X})$ is finite.
Then expected values of $\vector{h}$ under the squared family density may be understood in terms of expected values of a different function under the distribution obtained by normalising the base measure,
\begin{align*}
    \mathbb{E}_{\vector{\rv{x}} \sim p}[\vector{h}(\vector{\rv{x}})] &= \int_{\mathbb{X}}  \frac{\langle \matrix{M}\vector{\psi}(\vector{x}), \vector{\psi}(\vector{x}) \rangle}{\Tr\big( \matrix{M} \matrix{K}_{\mu, \vector{\psi}} \big)} \mu(\mathbb{X}) h(\vector{x}) \frac{\mu(d\vector{x})}{\mu(\mathbb{X})} \\
    &= \frac{\mu(\mathbb{X})}{\Tr\big( \matrix{M} \matrix{K}_{\mu, \vector{\psi}} \big)} \mathbb{E}_{\vector{\rv{z}} \sim \mu/\mu(\mathbb{X})}\Bigg[\langle \matrix{M}\vector{\psi}(\vector{\rv{x}}), \vector{\psi}(\vector{\rv{x}}) \rangle  \vector{h}(\vector{\rv{x}})\Bigg].
\end{align*}
\fi

\section{$g$-families}
\label{sec:g_fam}
We consider a general class of model families that include the squared family as a special case. 
The only purpose of this general class is to establish squared families and even-order monomial families as uniquely satisfying certain properties within this broader class.
In the context of these more general families, we generalise the convenient properties of squared families and define them as two desiderata.
We show that the first desideratum is satisfied \emph{only} by orthogonally singular $g$-families and exponential families.
We also characterise this more general family by observing that it is equivalent to positively homogeneous families, with even order families as a special case.
We demonstrate that the second desideratum is satisfied by even-order monomial families, but not for other positively homogeneous families or for exponential families.
Therefore in some sense squared families are an excellent representative of a class of families with convenient properties.

Let $\mathbbold{\Theta}$ be a convex parameter set and let $\mathbb{X} \subseteq \mathbb{R}^d$ be a domain.
Define a statistic $\vector{\psi}:\mathbb{X} \to \mathbb{R}^n$.
We consider probability density models $p(\cdot \mid \vector{\theta})$ which are proportional to evaluations of nonnegative transformations $g:\mathbb R\to [0,\infty)$ of a linear model $\vector{\theta}^\top \vector{\psi}(\cdot)$,
\begin{align*}
    p(\vector{x} \mid \vector{\theta}) &= \frac{g\big(\vector{\theta}^\top \vector{\psi}(\vector{x}) \big)}{z(\vector{\theta})}, \qquad z(\vector{\theta}) = \int_{\mathbb{X}} g\big(\vector{\theta}^\top  \vector{\psi}(\vector{x}) \big)\, \mu (d\vector{x}), \numberthis \label{eq:transformed_linear}
\end{align*}
parameterised by $\vector{\theta} \in \mathbbold{\Theta}$ and supported on a subset of $\mathbb{X}$, where $\mu$ is some appropriately defined base measure.
We call the set $\{ p(\cdot \mid \vector{\theta}) \}_{\vector{\theta} \in \mathbbold{\Theta}}$ a \emph{$g$-family}. 
Here the parameter space is a subset of all parameters which lead to a finite and nonzero normalising constant, $\mathbbold{\Theta} \subseteq \{ \vector{\theta} \in \mathbb{R}^n \mid 0 < z(\vector{\theta}) < \infty\}$.

\subsection{Two desiderata for probability model families}
\label{sec:desiderata}
\paragraph{Linked normalising constant, Hessian metric, Fisher information and divergence}
It is convenient if the normalising constant and its gradient and Hessian can also be directly applied to understand the geometry of the manifold of probability distributions in closed-form, without the requirement of computing any additional integrals. 
This means that only one unifying integral needs to be computed for both proper normalisation of the family of distributions and for understanding the limits of the estimation procedure.
Here we give a criteria for such a property to hold.
A Hessian metric is constructed as the Hessian of the convex Bregman generator $\phi$ of the parameter $\vector{\theta}$, and gives a Riemannian metric on the parameter space $\mathbbold{\Theta}$.
On the other hand, the Fisher information is the natural Riemannian metric on a manifold of probability distributions.
We desire that the convex generator depends on the parameter $\vector{\theta}$ only through the normalising constant $z(\vector{\theta})$, and the resulting metric is equal to a \emph{conformal transformation} of the Fisher information $\matrix{G}(\vector{\theta})$.
In other words, we require that
\begin{equation}
\label{eq:d1_}
\begin{alignedat}{3}
    &\exists \chi:(0, \infty) \to \mathbb{R}  \quad &&\text{ such that} \quad \phantom{\nabla^2}\phi(\vector{\theta}) &&= \chi\big(z(\vector{\theta})\big), \text{ and} \\
    &\exists c:(0, \infty) \to (0, \infty) \quad &&\text{ such that} \quad \nabla^2 \phi(\vector{\theta}) &&= c\big( z(\vector{\theta}) \big) \matrix{G}(\vector{\theta}). 
\end{alignedat}
\tag{Desideratum 1}
\end{equation}
Examples of families satisfying~\ref{eq:d1_} include exponential families, where $\chi=\log$, $g=\exp$ and $c=\text{Id}$, and $q$-exponential families~\citep{amari2011geometry}, which are not constructed as transformations $g$ of a linear model, but nevertheless admit a Fisher information which is conformally equivalent to a Hessian metric. 
Other examples are squared families, which satisfy~\ref{eq:d1_} as demonstrated in \S~\ref{sec:squared_fam}, with $\chi=\text{Id}$, $g=(\cdot)^2$ and $c=(\cdot)^{-1}$.

An additional property, related to~\ref{eq:d1_}, is that the Bregman divergence induced by the generator is equal to a statistical divergence over the family. This is related because the Hessian metric induced by the Bregman generator can be seen as the infinitesimal of the Bregman divergence, and if the statistical divergence were an $f$-divergence, the statistical divergence would be seen as the infinitesimal quadratic form of the statistical divergence.
Let $\overline{\mathbbold{\Theta}} \subseteq \mathbbold{\Theta}$ such that $\{p(\vector{x} \mid \vector{\theta}) \}_{\vector{\theta} \in \overline{\mathbbold{\Theta}}} = \{p(\vector{x} \mid \vector{\theta}) \}_{\vector{\theta} \in \mathbbold{\Theta}}$. 
That is, $\overline{\mathbbold{\Theta}}$ is a subset of the parameter space that does not restrict the family of probability distributions.
We also require that the Bregman divergence generated by $\phi$, when evaluated on $\overline{\mathbbold{\Theta}} \times \overline{\mathbbold{\Theta}}$ is a statistical divergence,
\begin{equation}
\label{eq:d2b}
    d_\phi( \vector{\theta}: \vector{\theta}') = \widetilde{D}(p: p'), \quad \forall  \vector{\theta}, \vector{\theta}' \in \overline{\mathbbold{\Theta}},
\end{equation}
for some statistical divergence $\widetilde{D}$.

\paragraph{Integral-parameter factorisation}
For some $M$ which is polynomial in $n$ and $d$, there exists a function $\vector{f}_1:\mathbbold{\Theta} \to \mathbb{R}^M$ independent of $\vector{\psi}$ and $\mu$, and a function $\vector{f}_2:\mathbb{R}^n \to \mathbb{R}^M$ independent of $\vector{\theta}$, such that the normalising constant can be represented as an evaluation of an inner product in $\mathbb{R}^M$,
\begin{align*}
    z(\vector{\theta}) = \Bigg\langle \vector{f}_1(\vector{\theta}), \int_{\mathbb{X}} \vector{f}_2\big( \vector{\psi}(\vector{x}) \big)\, \mu(d\vector{x}) \Bigg\rangle_{\mathbb{R}^M}. \tag{Desideratum 2} \label{eq:d2}
\end{align*}
One optional additional property to~\ref{eq:d2} is that $\int_{\mathbb{X}} \vector{f}_2\big( \vector{\psi}(\vector{x}) \big)\, \mu(d\vector{x})$ is available in closed-form.
Even without this additional property,~\ref{eq:d2} is attractive because one integral factor may be used to recompute the normalising constant of the entire family, since the normalising constant of each density belonging to the family only changes through the parameter factor. 
This is useful in designing computationally effective practical algorithms.
For example, if point estimates for parameters are sought using gradient-based optimisers, the integral factor need only be computed once, and parameter updates over the course of the gradient-based optimisation routine do not require recomputing a parameter-independent integral. 
As shown in \S~\ref{sec:squared_fam}, squared families satisfy~\ref{eq:d2}, with $\vector{f}_1$ and $\vector{f}_2$ both being the standard outer product in $\mathbb{R}^n$. 
To the best of our knowledge, no other previously studied family satisfies this integral-parameter factorisation.

\subsection{$g$-families via nonnegative transformations of linear models}

Given the two desiderata presented in Section~\ref{sec:desiderata}, and the fact that we have shown in \S~\ref{sec:squared_fam} that squared families satisfy both of them, a natural question is: which other families also satisfy both desiderata? It turns out that even-order monomial families (of which squared families are a prominent special case) satisfy both desiderata.
Exponential families do not obviously satisfy~\ref{eq:d2}.
More generally, only positively homogeneous families and exponential families satisfy~\ref{eq:d1_}.

\subsubsection{Regularity conditions}
In this section, we will utilise some mild regularity conditions on $g$, $\vector{\psi}$ and $\vector{\mu}$, as detailed in Assumptions~\ref{ass:reg_g} and~\ref{ass:reg_psi}.
It is likely that some of these regularity conditions can be relaxed in future work.
We first state our regularity requirements on $g$.
\begin{assumption}
\label{ass:reg_g}
\mbox{}
\begin{itemize}
    \item 
    $g:\mathbb{R} \to [0,\infty)$, and takes values of $0$ at a countable number of points.
    \item The function $g$ is measurable (with respect to $\mu$).
    \item The function $g$ is twice continuously differentiable. 
    \item $g(1)=1$.
\end{itemize}
\end{assumption}
The first condition is conservatively strong and avoids having to define some complicated support in terms of the range of $\vector{\psi}$ and the parameter set $\mathbbold{\Theta}$. 
The second condition is a strict requirement for applying standard tools from probability theory.
The third condition allows for various classical tools surrounding the Fisher information and asymptotic normality of estimators to be applied, although it might be possible to relax with the more advanced notion of local asymptotic normality~\citep[for example]{LeCam2000}.
The last condition is without loss of generality and fixes the scale of $g$, since multiplying any $g$ by a constant results in the same density.
We now state our regularity requirements on $\vector{\psi}$.
\begin{assumption}
\label{ass:reg_psi}
\mbox{}
\label{ass:feature_regularity}
\begin{itemize}
    \item  The statistic $\vector{\psi}$ is not $\mu$-almost everywhere proportional to a constant, that is, there does not exist a $\psi:\mathbb{X} \to \mathbb{R}$ and $\vector{a}\in \mathbb{R}^n$ such that $\vector{\psi}(\vector{x}) = \psi(\vector{x}) \vector{a}$ for some set with nonzero measure (with respect to $\mu$).
    %\item For any two $\vector{\theta}, \vector{\theta}' \in \mathbbold{\Theta}$, the intersection $\{ \vector{x} \mid \vector{\theta}'^\top \vector{\psi}(\vector{x}) \neq 0\} \cap \{ \vector{x} \mid \vector{\theta}^\top \vector{\psi}(\vector{x}) \neq 0\}$ has nonzero measure with respect to $\mu$.
    \item For any $\vector{\theta}$ the set $\mathbb{A}_{\vector{\theta}} = \{\vector{\theta}^\top \vector{\psi}(\vector{x}) \mid \vector{x} \in \mathbb{X}\} \subseteq \mathbb{R}$ contains an open subset containing $0$.
    \item The last dimension $\psi_n(\vector{x})$ is constant (say $1$), and the last coordinate of the parameter space is $\mathbb{R}$.
\end{itemize}
\end{assumption}
The first condition rules out features $\vector{\psi}$ which always point in the same direction, i.e. features whose useful information amounts to their norm. 
The second condition essentially requires that any predictor $\vector{\theta}^\top \vector{\psi}(\vector{x})$ can at some point in $\mathbb{X}$ become negative.
The final condition is similar to having a bias in a neural network. 
This bias is helpful in ensuring the model is sufficiently rich to capture all functions of interest.

\subsection{Orthogonal singularity and Fisher information}
\label{sec:g_singular}
For a $g$-family of the form~\eqref{eq:transformed_linear}, as shown in Appendix~\ref{app:fisher_gfam}, the Fisher information~\eqref{eq:FIM} always admits a decomposition as a difference of a PSD matrix and a rank-$1$ PSD matrix,
\begin{align*}
    \matrix{G}_{p}(\vector{\theta}) &= \mathbb{E}_{p(\cdot \mid \vector{\theta})} \Big[ \frac{\vector{\psi}(\vector{\rv{x}}) \vector{\psi}(\vector{\rv{x}})^\top g'\big( \vector{\theta}^\top \vector{\psi}(\vector{\rv{x}}) \big)^2}{g\big( \vector{\theta}^\top \vector{\psi}(\vector{\rv{x}}) \big)^2 } \Big]  -  \frac{1}{ z(\vector{\theta})^2} \nabla z(\vector{\theta}) \nabla z(\vector{\theta})^\top. \numberthis \label{eq:FIM2}
\end{align*}
Even if the first term is strictly positive definite, the second term is at times large enough to make the resulting $\matrix{G}(\vector{\theta})$ have a zero eigenvalue (and more generally, the first term need not be strictly positive definite).

We call a $g$-family $\{ p(\cdot \mid \vector{\theta}) \}_{\vector{\theta} \in \mathbbold{\Theta}}$ \emph{orthogonally singular} if for all $\theta \in \mathbbold{\Theta}$ and $\vector{x} \in \mathbb{X}$,
\begin{align*}
\vector{\theta}^\top \nabla_{\vector{\theta}} \log p(\vector{x} \mid \vector{\theta}) = 0. \numberthis \label{eq:param_score_orth}
\end{align*}
Since the Fisher information is the expected outer product of scores,~\eqref{eq:param_score_orth} implies that $\vector{\theta}$ is an eigenvector of the Fisher information associated with an eigenvalue $0$.

\subsubsection{Orthogonally singular $g$-families as positively homogeneous families}
We provide an alternate characterisation of orthogonally singular $g$-families as positively homogeneous families.

\begin{restatable}{lemma}{poshomo}
\label{lemma:even_order}
    Suppose Assumption~\ref{ass:reg_g} holds. Then 
    $\{ p(\cdot \mid \vector{\theta}) \}_{\vector{\theta} \in \mathbbold{\Theta}} $ is an orthogonally singular $g$-family if and only if $g$ positively homogeneous.
    That is, there exists some $c>0$ and $k \geq 2$ such that 
    \begin{align*}
    g(a) = \begin{cases} \phantom{c}|a|^k, \quad a > 0 \\
    c |a|^k, \quad a \leq 0. \end{cases} \numberthis \label{eq:pos_hom}
\end{align*}
\end{restatable}
The case $k=0$ is excluded because this implies that the family is a singleton with a parameter that does not have any effect, and the Fisher information is exactly $\matrix{0}$.

\begin{remark}[Bias reparameterisation]
    For any $c_2 \in \mathbb{R}$, since the last dimension of $\vector{\psi}$ is $1$ and the last coordinate of the parameter space is allowed to be any element of $\mathbb{R}$, we may reparameterise $g$-families to explicitly include a constant bias term $c_2 \in \mathbb{R}$. 
    For positively homogeneous families, this reparameterisation amounts to
    \begin{align*}
    g(a) = \begin{cases} \phantom{c_1}|a + c_2|^k, \quad a + c_2 > 0 \\
    c_1 |a + c_2|^k, \quad a + c_2 \leq 0. \end{cases}
\end{align*}
The orthogonal singularity~\eqref{eq:param_score_orth} reparameterises as $\big(\vector{\theta}+(0;,\ldots;0;c_2)  \big)^\top \nabla_{\vector{\theta}} \log p(\vector{x} \mid \vector{\theta}) = 0$, which may be verified directly or by noting that the Jacobian of the transformation associated with $\vector{\theta} \mapsto \big(\vector{\theta}+(0;,\ldots;0;c_2)  \big)$ is the identity matrix.
\end{remark}

\begin{remark}
    The only infinitely differentiable positively homogeneous families are even-order monomial families, $g(a) = a^{k}$ where $k \in \{2, 4, \ldots\}$.
\end{remark}

\subsubsection{Removing orthogonal singularity with dimension-augmentation}
\label{sec:dim_aug}

As a consequence, combining~\eqref{eq:FIM2} with~\eqref{eq:additive_fisher}, we can attempt to cancel out the rank-$1$ negative contribution in~\eqref{eq:FIM2}.
We call $p(\vector{x},a  \mid \vector{\theta}) = p(\vector{x} \mid \vector{\theta}) p(a \mid \vector{\theta})$ where $p(\vector{x} \mid \vector{\theta})$  belongs to a $g$-family and $ p(a \mid \vector{\theta}) = \mathcal{N}\big(a \mid \log z(\vector{\theta}), \sigma^2\big)$ a \emph{dimension-augmented $g$-family}.
\begin{restatable}{lemma}{gaussian_fim}
    \label{lemma:gaussian_fim}
    Suppose Assumption~\ref{ass:reg_g} holds. 
    The Fisher information of a univariate Gaussian with mean $\log z(\vector{\theta})$ and variance $\sigma^2>0$ is given by the rank-$1$ matrix $\frac{1}{\sigma^2 z(\vector{\theta})^2} \nabla z(\vector{\theta}) \nabla z(\vector{\theta})^\top$.
    Let $p_1(\vector{x} \mid \vector{\theta})$ belong to a $g$-family and $ p_2(a \mid \vector{\theta}) = \mathcal{N}\big(a \mid \log z(\vector{\theta}), \sigma^2\big)$. 
    Consequently, the Fisher information of the dimension-augmented density $p_1(\vector{x} \mid \vector{\theta}) p_2(a \mid \vector{\theta})$ is
    $$ \matrix{G}_{p_1(\cdot) p_2(\cdot)}(\vector{\theta}) = \mathbb{E}_{p(\cdot \mid \vector{\theta})} \Big[ \frac{\vector{\psi}(\vector{\rv{x}}) \vector{\psi}(\vector{\rv{x}})^\top g'\big( \vector{\theta}^\top \vector{\psi}(\vector{\rv{x}}) \big)^2}{g\big( \vector{\theta}^\top \vector{\psi}(\vector{\rv{x}}) \big)^2 } \Big]  +  \frac{\sigma^{-2} - 1}{ z(\vector{\theta})^2} \nabla z(\vector{\theta}) \nabla z(\vector{\theta})^\top. $$
\end{restatable}
Note that taking $\sigma \to \infty$ recovers the Fisher information of the general $g$-family (with no dimension-augmentation).
The marginal distribution $p(\vector{x} \mid \vector{\theta})$ of any dimension-augmented $g$-family is a $g$-family. 
In this sense, dimension-augmented $g$-families generalise $g$-families.
\begin{restatable}{lemma}{dimaugregular}
    \label{lemma:dim_aug_regular}
        Suppose Assumptions~\ref{ass:reg_g} and~\ref{ass:reg_psi} hold and $\sigma^{-2} \geq 1$.
    A dimension-augmented orthogonally singular $g$-family is regular. That is, a dimension-augmented positively homogeneous family is regular.
\end{restatable}

%It is much easier to ensure dimension-augmented $g$-families have a non-singular Fisher information compared with their general $g$-family counterparts. Unfortunately
In practice, constructing a sample from $p(\vector{x},a  \mid \vector{\theta})$ belonging to a dimension-augmented $g$-family given a sample from $p(\vector{x}  \mid \vector{\theta})$ belonging to a $g$-family requires knowing or otherwise constraining the normalising constant $z(\vector{\theta})$.
This is not a problem for positively homogeneous families, because $p(\vector{x} \mid \vector{\theta}) = p(\vector{x} \mid s \vector{\theta})$ and $z(s \vector{\theta}) = |s|^k z(\vector{\theta})$ for any $s > 0$, so $z$ may be chosen or constrained arbitrarily without imposing a condition on the marginal $g$-family  $p(\vector{x}  \mid \vector{\theta})$.

%\subsection{Computational and computational-geometric properties}
Finally, we find that of all dimension-augmented families, only exponential families and $g$ families have a linked normalising constant and Fisher information via~\ref{eq:d1_}. 
\begin{restatable}{theorem}{characterisation}
    \label{thm:characterisation}
Suppose Assumptions~\ref{ass:reg_g} and~\ref{ass:reg_psi} hold. 
Exponential families and dimension-augmented positively homogeneous families are the only dimension-augmented $g$-families which satisfy~\ref{eq:d1_} for all $\mu$ such that $\mu(\mathbb{X}) < \infty$. For dimension-augmented positively homogeneous families,~\ref{eq:d1_} reads as
\begin{align*}
    \phi(\vector{\theta}) = z(\vector{\theta}) \qquad \nabla^2 \phi(\vector{\theta}) = \frac{k-1}{k} z(\vector{\theta}) \matrix{G}(\vector{\theta}).
\end{align*}
\end{restatable}

Through the proof of Theorem~\ref{thm:characterisation}, we see that positively homogeneous families under the bias reparameterisation are very closely related to $q$-exponential families~\citep{naudts2004estimators,amari2011geometry,naudts2011generalised}, so it is helpful to clarify their differences.
There are two obvious ways to generalise the normalisation condition of exponential families: (1) by defining a normalising constant as the integral of $g$, as we have done with $g$-families, and (2) by implicitly defining a generalised log normalising function $A$ as a function such that $g\big( \vector{\theta}^\top \vector{\psi}(\vector{x}) - A(\vector{\theta})\big)$ integrates to $1$, as is done with $q$-exponential families and other deformed exponential families~\citep{naudts2011generalised}.
Such implicit definitions may cause an additional layer of intractability of the log normalising function --- it is no longer necessarily even expressed as an explicit (potentially still intractable) integral.
The discrete $q$-exponential families have parameter metrics which are conformally equivalent to the Fisher information~\citep[Theorem 4]{amari2011geometry}, however their conformal transformation is defined in terms of a so-called escort distribution rather than depending on the parameter only through the normalising constant.
Finally, we note that without dimension-augmentation, positively homogeneous families are always singular, whereas $q$-exponential families may not be singular.
Indeed, it is the controlled singularity that allows us to search beyond exponential families for $g$-families which satisfy~\ref{eq:d1_} while still allowing for estimation and analysis of estimation procedures.

\subsection{Normalising constant}
Theorem~\ref{thm:characterisation} identifies exponential families and dimension-augmented positively homogeneous families as special, in the sense that they are the only $g$-families which satisfy~\ref{eq:d1_}. 
When we additionally consider~\ref{eq:d2}, we can further remove exponential families and many positively homogeneous families.
\begin{remark}
    Even-order monomial families, a special case of positively homogeneous families~\eqref{eq:pos_hom} with $k$ an even integer and $c=1$,  satisfy~\ref{eq:d2}. 
    Using multiindex notation,
    \begin{align*}
    z(\vector{\theta}) &= \sum_{|\vector{\alpha}| = k} {k \choose \vector{\alpha}} \vector{\theta}^{\vector{\alpha}} \int_{\mathbb{X}} \vector{\psi}(\vector{x})^{\vector{\alpha}} \, \mu(d\vector{x}).
    \end{align*}
    Equivalently, using $k$-fold tensor product notation,
    \begin{align*}
        z(\vector{\theta})=\left\langle  \vector{\theta}^{\otimes k},\int_{\mathbb{X}}\vector{\psi}(\vector{x})^{\otimes k}\, \mu(d\vector{x})  \right\rangle. 
    \end{align*}
    For exponential families, a representation in the form of~\ref{eq:d2} does not obviously follow, because the exponential function admits an infinite monomial series representation. 
    We are not aware of any useful exponential families which admit an integral-parameter decomposition as in~\ref{eq:d2}.
    Similarly, general positively homogeneous families do not admit an obvious parameter-integral factorisation. 
\end{remark}

\iffalse
\ds{Is it easier to understand the above as: 
For $g(u)=u^{k}$, we have
\[
z(\boldsymbol{\theta})=\left\langle  \boldsymbol{\theta}^{\otimes k},\int_{\mathbb{X}}\boldsymbol{\psi}(\mathbf{x})^{\otimes k}\mu(d\mathbf{x})  \right\rangle. 
\]}
\fi

\section{Parameter estimation and density estimation}
\label{sec:estimation}
Using our derived geometry of squared families, we can analyse the errors of statistical estimation procedures. 
Here we examine arguably the most prevalent procedure, maximum likelihood estimation.
We do this in three contexts of increasing difficulty: a well-specified model, a misspecified model, and a setting where we aim to estimate an arbitrary target density using a universal approximator.

\subsection{Maximum likelihood estimation}
Let $\{ (\vector{x}_i, a_i)\}_{i=1}^N$ be sampled from some probability measure $q(\vector{x}) \mathcal{N}(a\mid 0, 1) \mu(d\vector{x}) da$, where $q$ belongs to a squared family $\{ p(\cdot\mid \vector{\theta})\}_{\vector{\theta} \in \mathbbold{\Theta}}$ with true parameter $\vector{\theta}_\ast\in \overline{\mathbbold{\Theta}}$, and consider the maximum likelihood estimator
\begin{align*}
    \hat{\vector{\theta}}_{N} &=   \argmin_{\vector{\theta}' \in \hat{\mathbbold{\Theta}}} \sum_{i=1}^N -\log p(\vector{x}_i \mid \vector{\theta}') + \frac{1}{2}(a_i - \log  \vector{\theta}'^\top \matrix{K}_{\mu, \vector{\psi}} \vector{\theta}')^2. \numberthis \label{eq:mle}
\end{align*}
%We specify some $\theta_1' \geq \varepsilon > 0$ in order to ensure the parameter space is compact so that standard textbook results on asymptotically Gaussian precision of maximum likelihood estimates apply, but this can likely be relaxed to $\theta_1' > 0$.
Here $\hat{\mathbbold{\Theta}} = \{ \vector{\theta} \in \mathbbold{\Theta} \mid \theta_1 \geq \epsilon > 0, \vector{\theta}^\top \matrix{K}_{\mu, \vector{\psi}} \vector{\theta} \leq R \}$ is a compact subset of $\mathbbold{\Theta}$, and $R > 1$. We use a compact space so that standard textbook results on asymptotically Gaussian precision of maximum likelihood estimates apply, but this can likely be relaxed.
Note that the constraint $\vector{\theta}^\top \matrix{K}_{\mu, \vector{\psi}} \vector{\theta} \leq R$ imposes no condition on the squared family in distribution space. The constraint $\theta_1 \geq \epsilon > 0$ can be made arbitrarily close to $\theta_1 > 0$, which would impose no condition on the squared family in distribution space.

Given data $\vector{x}_i$, such an optimisation objective can always be formed by artificially augmenting data $\vector{x}_i$ from $q$ with data $a_i$ from an independent standard Gaussian.
It is interesting to note that dimension augmentation has a regularisation-like effect, where the second term in~\eqref{eq:mle} acts like a regulariser which encourages the normalising constant to be close to $1$ with random perturbations.

It is well-known that the MLE satisfies asymptotic normality, in the sense that 
\begin{align*}
    \sqrt{N}\big( \hat{\vector{\theta}}_N - \vector{\theta}_\ast \big) \stackrel{d}{\longrightarrow} \mathcal{N}\Big(0, \frac{1}{4} \matrix{K}_{\mu,\vector{\psi}}^{-1}  \Big).
\end{align*}
Here we have used the closed-form expression of the Fisher information, noting that $z(\vector{\theta}_\ast) = 1$.

\subsection{Maximum likelihood estimation under model misspecification}
Quasi-maximum likelihood estimation is distinguished from maximum likelihood estimation when the density $q$ from which data is sampled does not belong to the family of probability distributions from which the best estimate is taken.
Let $\{ (\vector{x}_i, a_i)\}_{i=1}^N$ be sampled from some probability measure $q(\vector{x}) \mathcal{N}(a\mid 0, 1) \mu(d\vector{x}) da$, and consider the quasi maximum likelihood estimator
\begin{align*}
    \hat{\vector{\theta}}_{N} &=   \argmin_{\vector{\theta}' \in \hat{\mathbbold{\Theta}} } \sum_{i=1}^N -\log p(\vector{x}_i \mid \vector{\theta}') + \frac{1}{2}(a_i - \log  \vector{\theta}'^\top \matrix{K}_{\mu, \vector{\psi}} \vector{\theta}')^2. \numberthis \label{eq:qmle}
\end{align*}
%Note that no matter the choice of $\lambda$, the induced distribution $p(\vector{x}_i \mid \hat{\vector{\theta}}_N)$ does not change, but $\lambda$ controls the normalising constant and sets a scale.
We also define the infinite sample version of~\eqref{eq:qmle}, the parameter of the KL projection
\begin{align*}
    \vector{\theta}_\ast &= \argmin_{\vector{\theta}' \mid \theta_1' \geq \varepsilon > 0}  \kl [q(\vector{x}): p(\vector{x}_i \mid \vector{\theta}')  ] + \kl [\mathcal{N}(a\mid 0, 1): \mathcal{N}(a\mid \log  \vector{\theta}'^\top \matrix{K}_{\mu, \vector{\psi}} \vector{\theta}' , 1)] \numberthis \label{eq:kl_minimiser} \\
    &= \argmin_{\substack{\theta_1' \geq \varepsilon > 0 \\  \vector{\theta}'^\top \matrix{K}_{\mu, \vector{\psi}} \vector{\theta}' = 1}}  \kl [q(\vector{x}): p(\vector{x}_i \mid \vector{\theta}')], 
\end{align*}
where equality is due to the fact that due to absolute homogeneity, $\vector{\theta}'^\top \matrix{K}_{\mu, \vector{\psi}} \vector{\theta}'$ can be chosen independently of $p(\vector{x}_i \mid \vector{\theta}')$, and in particular can be chosen to equal $1$ without imposing a constraint in model space, therefore minimising the second term in the objective. 
In order to analyse this setting, we assume that the infinite data variant is well-posed.
\begin{assumption}
\label{ass:well_posed}
    A unique minimiser $\vector{\theta}_\ast$ of~\eqref{eq:kl_minimiser} exists in the interior of the constraint set, $\vector{\theta}_\ast \in \hat{\mathbbold{\Theta}}$.
\end{assumption}
We also assume that the second moments of the feature under the likelihood ratio between the target and the KL projection are finite.
\begin{assumption}
\label{ass:ratio}
    Define a ratio base measure $r(\vector{x}) \mu (d\vector{x}) = \frac{q(\vector{x})}{ p(\vector{x}' \mid \vector{\theta}_\ast)} \mu(d\vector{x})$, and suppose that 
    \begin{align*}
        \matrix{K}_{r(\cdot) \mu, \vector{\psi}} = \int_{\mathbb{X}} \psi(\vector{x}) \psi(\vector{x})^\top r(\vector{x}) \mu(d\vector{x})
    \end{align*}
    exists and is finite.
\end{assumption}

\begin{restatable}{theorem}{mlemisspecified}
    Suppose Assumptions~\ref{ass:strictly_pd},~\ref{ass:well_posed} and~\ref{ass:ratio} hold. The following convergence in probability holds,
    \begin{align*}
        \sqrt{N} \Big| \kl\big(q:p(\cdot\mid \hat{\vector{\theta}}_N) \big) - \kl\big(q:p(\cdot\mid \vector{\theta}_\ast) \big) \Big|  \stackrel{p}{\longrightarrow} 0.
    \end{align*}
\end{restatable}
As a corollary, when the model is well-specified, $q(\cdot) = p(\cdot\mid \vector{\theta}_\ast)$ and $\sqrt{N} \kl\big(q:p(\cdot\mid \hat{\vector{\theta}}_N) \big) \stackrel{p}{\longrightarrow} 0$.

\subsection{Universal approximation}
In the case where the model is misspecified, given the universal approximating property of certain feature extractors, one might expect that for large enough $n$, a small projected KL divergence $\kl\big(q:p(\cdot\mid \vector{\theta}_\ast) \big) $ can be obtained.
This is indeed the case, as we now discuss. 
In order to do so, we define the notion of a universal approximator.
\begin{assumption}[$\mathcal{F}$-Universal approximation property]
\label{ass:universal_approx}
    The set of functions $\omega$ parameterised by $\vector{\theta}$ of the form $\omega(\vector{x}) = \vector{\theta}^\top \vector{\psi}(\vector{x})$ satisfies a universal approximation property, as follows.
    Let $\mathcal{F}$ be a rich set of functions.
    For any $f \in \mathcal{F}$, there exists a $\vector{\theta}_{\ast\ast}$ with $\Vert \vector{\theta}_{\ast\ast} \Vert_2^2 \leq c$ such that
    \begin{align*}
        \sqltwo(\omega^{\ast\ast}, f) \leq C/n,
    \end{align*}
    where $\omega^{\ast\ast} = {\vector{\theta}_{\ast\ast}}^\top \vector{\psi}$ and $C > 0$.
\end{assumption}
An early work which uncovers architectures satisfying Assumption~\ref{ass:universal_approx} is that of~\citet{barron1993universal}, which looks at linear combinations of sigmoidal functions with random hidden parameters. 
However, the relationship between the set of functions $\mathcal{F}$ which can be approximated and the distribution on the random hidden parameters was not so clear.
More recently,~\citet{gonon2023approximation} find that a very wide range of classes $\mathcal{F}$ can be approximated using uniformly distributed random parameters and more general activation functions, including the ReLU example.
Another classical example includes shallow random neural networks~\citep[implied by Lemma 1]{rahimi2008weighted}, which we include below for concreteness.
\begin{lemma}
    Let $\vector{\phi}_{\vector{w}}$ be a random feature mapping parameterised by parameters $\vector{w}$ satisfying $\sup_{\vector{x} \in \mathbb{X}} | \vector{\phi}_{\vector{w}}(\vector{x}) |\leq 1$. 
    Let $\rho:\mathbb{P} \to (0, \infty)$ be some probability density function supported on $\mathbb{P}$.
    For some $C > 0$, define the sets of functions
    \begin{align*}
        \mathcal{F} &= \Bigg\{ f \Big| f(\vector{x}) = \int_{} \alpha(\vector{w}) \phi_{\vector{w}}(\vector{x}) \, d\vector{w}, \quad | \alpha(\vector{w}) | \leq C \rho(\vector{w})  \Bigg\} \\
        \hat{\mathcal{F}} &= \Bigg\{ f \Big| f(\vector{x}) = \sum_{i=1}^n \frac{\alpha_i}{\sqrt n} \frac{\phi_{\vector{\rv{w}}_i}(\vector{x})}{\sqrt n},  \quad | \alpha_i | \leq C, \quad \vector{\rv{w}}_i \sim \rho \Bigg\},
    \end{align*}
    the latter being a random variable. 
    Let $f^\ast \in \mathcal{F}$ be a target function.
    Then for any $\delta > 0$, with probability at least $1 - \delta$, there exists an approximating function $\hat{f} \in \hat{\mathcal{F}}$ such that
    \begin{align*}
        \sqltwo(\hat{f}, f^\ast) \leq \frac{C^2}{n} \Bigg( 1 + \sqrt{2\log \frac{1}{\delta}}\Bigg)^2.
    \end{align*}
\end{lemma}
Examples include random Fourier feature models, where $\vector{\phi}_{\vector{w}}$ are cosines with frequency $\vector{w}$, $\mathcal{F}$ is the set of all functions whose Fourier transforms decay faster than $C \rho(\vector{w})$, and the approximating class $\hat{\mathcal{F}}$ is a single hidden layer fully connected network with random cosine activations.
Choosing parameters $v_i = \frac{\alpha_i}{\sqrt n}$ and random features $\psi_i(\vector{x}) = \frac{\phi_{\vector{\rv{w}}_i}(\vector{x})}{\sqrt n}$, it is evident that $\hat{\mathcal{F}}$ describes a class of models which when squared and appropriately normalised, are a restricted class of a squared family. 
Using the universal approximation property, we are able to show the following.
%https://projecteuclid.org/journals/annals-of-statistics/volume-19/issue-3/Approximation-of-Density-Functions-by-Sequences-of-Exponential-Families/10.1214/aos/1176348252.full    <- similar thing but with exponential families, and what seems to be a better bound.
\begin{restatable}{theorem}{universalapprox}
    Suppose that $1\leq \mu(\mathbb{X}) < \infty$.
Suppose Assumptions~\ref{ass:strictly_pd},~\ref{ass:well_posed} and~\ref{ass:ratio} hold.
Suppose that Assumption~\ref{ass:universal_approx} holds, with $\sqrt{q} \in \mathcal{F}$.
Suppose $q(\vector{x}) < q_{\text{max}} < \infty$.
    There exists some constant $C$ such that 
    $\kl\big(q:p(\cdot\mid \vector{\theta}_\ast) \big) \leq C n^{-1/4}$.
    As a result, there exists some data-independent function $s(n)$ such that
    $$\sqrt{N} \Big| \kl\big(q:p(\cdot\mid \hat{\vector{\theta}}_N) \big) - s(n) \Big|  \stackrel{p}{\longrightarrow} 0,$$
    with $s(n) \leq C n^{-1/4}$.
\end{restatable}

Learning of arbitrary densities under maximum likelihood estimation has similarly been performed for exponential families in one dimension $\mathbb{X} = \mathbb{R}$, where hidden layers (or more appropriately, sufficient statistics) $\vector{\psi}$ are taken to be polynomial, trigonometric or splines~\citep{barron1991approximation}.
Here the log of the target density is assumed to belong to a square integrable Sobolov space. 
A rate of the KL divergence from the target density $q$ to the maximum likelihood of the exponential family converges in probability at a rate of $\mathcal{O}\Big( n^{-2v}+ n/N \Big)$ is obtained, where $v$ is a smoothness parameter.
Here $\mathcal{O}\big( n^{-2v} \big)$ captures the approximation error, and $\mathcal{O}(n/N)$ captures the estimation error.
This rate may appear favourable to density estimation using squared families, however we note that it only applies for one-dimensional distributions with bounded support (say, $[0,1]$).  
Also, a parameter-integral factorisation does not hold, so computing the normalising constant requires more effort than in squared families. 
Also, the limiting regime requires $n \to \infty$ and $N \to \infty$, with $n^2/N \to 0$ and $n^3 / N \to 0$ respectively for splines/trigonometrics and polynomials.
In contrast, our squared families result requires only $N \to \infty$, and the number of parameters $n$ is not tied to the number of datapoints.
The reason for this difference is that in exponential families, the estimation error depends on both $n$ and $N$, whereas surprisingly, in squared families, the estimation error depends only on $N$ and not $n$.
Finally, less formally, squared families can use arbitrary pretrained neural network features, while still admitting a tractable approximation of the normalising constant via~\ref{eq:d1_}. 
While it is not easy to describe the set of functions which can be approximated by arbitrary pretrained neural networks, in practice, they are very often used as rich function approximators where only the last layer is trained, a procedure known as fine-tuning foundation models.

\section{Conclusion}
Here we studied squared families as unique special cases of even-order monomial families, which are in turn special cases of positively homogeneous families and $g$-families. 
Positively homogeneous families are \emph{characterised} by their singular nature (Lemma~\ref{lemma:even_order}), and their singularity can be handled in a straightforward way with dimension augmentation (Lemma~\ref{lemma:gaussian_fim}).
Once the singularity has been handled, it is shown that exponential families and positively homogeneous monomial families are the only $g$-families with a Fisher information which is conformally equivalent to a Hessian metric generated by a Bregman divergence which depends only on the normalising constant (Theorem~\ref{thm:characterisation}). 
This computational-geometric property means that for exponential and positively homogeneous families, only one integral in the normalising constant is required to compute the Fisher information.
Furthermore, a powerful computational property in the parameter-integral factorisation simplifies the calculation of the normalising constant for even-order families.

Instances of squared family models have appeared recently in somewhat disparate subfields of machine learning: Kernel methods and Gaussian processes~\citep{marteau2020non,rudi_ciliberto,marteau2022sampling}, Neural networks~\citep{tsuchida2023squared,tsuchida2024snefy-ppp}, and probabilistic circuits~\citep{Sladek2023,loconte2023negative}, showing remarkable representation, estimation and marginalisation properties.
Despite this, surprisingly, squared families have not previously been analysed in terms of standard and powerful information-geometric and statistical frameworks.
Perhaps the biggest obstacle to this analysis was the fact that squared models are singular, and therefore it might have seemed analytically involved to perform such analysis.
Here we showed that the simple technique of dimension-augmentation can be used to convert singular squared families to regular families.
For squared families, a statistical divergence linked to the normalising constant is also found.
Under squared families, statistical estimation of well-specified and misspecified models is studied, as well as density estimation using universal approximation properties. 
We believe that squared families offer a powerful new way forward for density modelling using deep neural networks.
Unlike exponential families with neural network statistics or energy-based models --- both of which have intractable normalising constants ---
squared families' powerful parameter-integral decomposition and closed form links between normalising constants, divergences and Fisher information, offer a new way forward for density estimation with deep learning.

\clearpage
\acks{We would like to thank Jeremias Knoblauch and Frank Nielsen for helpful discussions and pointers. }

% Manual newpage inserted to improve layout of sample file - not
% needed in general before appendices/bibliography.

\newpage
\vskip 0.2in
\bibliography{sample}

\clearpage

\appendix
\section{Proofs for \S~\ref{sec:squared_fam}}
\divergencevec*
\begin{proof}
    This is a special case of Theorem~\ref{thm:divergencemat}.
\end{proof}

\divergencemat*
\begin{proof}
We first show that the Bregman divergence generated by $\phi$ is twice the squared $L^2$ distance.
    The Bregman divergence is
    \begin{align*}
        d_\phi[\vector{\theta}: \vector{\theta}'] &=  \vector{\theta}^\top \matrix{K}_{\mu, \vector{\psi}}^{(m)} \vector{\theta} -  \vector{\theta}'^\top \matrix{K}_{\mu, \vector{\psi}}^{(m)} \vector{\theta}' - 2\big(\matrix{K}_{\mu, \vector{\psi}}^{(m)} \vector{\theta}\big)^\top \big( \vector{\theta} - \vector{\theta}'\big) \\
        &= \big( \vector{\theta} - \vector{\theta}'\big)^\top \matrix{K}_{\mu, \vector{\psi}}^{(m)}\big( \vector{\theta} - \vector{\theta}'\big)^\top \\
        &= \int_{\mathbb{X}} \big( \vector{\theta} - \vector{\theta}'\big)^\top \big( \vector{\psi}(\vector{x}) \vector{\psi}(\vector{x})^\top \otimes \matrix{I}_{m \times m} \big)\big( \vector{\theta} - \vector{\theta}'\big)^\top \, d\mu(\vector{x}) \\
        &= \int_{\mathbb{X}}  \Vert (\matrix{\Theta} - \matrix{\Theta}') \vector{\psi}(\vector{x}) \Vert_2^2 \, d\mu(\vector{x}) \\
        &= 2 \sqltwo\big( \matrix{\Theta} \vector{\psi}(\cdot): \matrix{\Theta}' \vector{\psi}(\cdot) \big).
    \end{align*}

    \paragraph{Statistical divergence}
    We now show that $d_\phi[\vector{\theta}_1: \vector{\theta}_2]$ when evaluated on $\overline{\mathbbold{\Theta}} \times \overline{\mathbbold{\Theta}}$ is a statistical divergence on the squared family $\{ p(\vector{x} ; \vector{\theta})\}_{\vector{\theta} \in \mathbbold{\Theta}}$, by showing it satisifes the required three axioms.
    Nonnegativity is immediate.
Second, for positivity, suppose $d_\phi[\vector{\theta}_1:\vector{\theta}_2] = 0$ which is equivalent to $(\vector{\theta}_1 - \vector{\theta}_2)^\top \matrix{K}_{\mu, \vector{\psi}}^{(m)} (\vector{\theta}_1 - \vector{\theta}_2) = 0$. 
If $\matrix{K}_{\mu, \vector{\psi}}^{(m)}$ has a Cholesky decomposition $\matrix{K}_{\mu, \vector{\psi}}^{(m)} = \matrix{R} \matrix{R}^\top$, then this is equivalent to $\matrix{R}^\top(\vector{\theta}_1 - \vector{\theta}_2) = 0$. 
Finally, since $\matrix{K}_{\mu, \vector{\psi}}^{(m)}$ is full-rank, then $\matrix{R}$ is invertible and this is equivalent to $\vector{\theta}_1 = \vector{\theta}_2$, which implies that $p(\vector{x}\mid \vector{\theta}_1) = p(\vector{x} \mid \vector{\theta}_2)$. 
Now for the opposite direction, suppose $p(\vector{x}\mid \vector{\theta}_1) = p(\vector{x} \mid \vector{\theta}_2)$. 
Then $\Tr\big(\matrix{\Theta}_1^\top \matrix{\Theta}_1  \vector{\psi}(\vector{x}) \vector{\psi}(\vector{x})^\top \big) = \Tr\big(\matrix{\Theta}_2^\top \matrix{\Theta}_2  \vector{\psi}(\vector{x}) \vector{\psi}(\vector{x})^\top \big)$. Rearranging, this is equivalent to  $\vector{\psi}(\vector{x})^\top (\matrix{\Theta}_1^\top \matrix{\Theta}_1 - \matrix{\Theta}_2^\top \matrix{\Theta}_2) \vector{\psi}(\vector{x}) = 0$. 
If the span of $\{ \vector{\psi}(\vector{x}) \mid \vector{x} \in \mathbb{X}\}$ is all of $\mathbb{R}^n$, then this implies $\vector{z}^\top (\matrix{\Theta}_1^\top \matrix{\Theta}_1 - \matrix{\Theta}_2^\top \matrix{\Theta}_2) \vector{z} = 0$ for all $\vector{z} \in \mathbb{R}^n$. Since $\matrix{\Theta}_1^\top \matrix{\Theta}_1 - \matrix{\Theta}_2^\top \matrix{\Theta}_2$ is symmetric, this is true if and only if $\matrix{\Theta}_1^\top \matrix{\Theta}_1 = \matrix{\Theta}_2^\top \matrix{\Theta}_2$, and since both the Cholesky decomposition and the rank-$1$ decompositions are unique, it follows that $\matrix{\Theta}_1 = \matrix{\Theta}_2$.
Third, for positive definiteness of the quadratic form, we have already shown that $\nabla^2 \phi(\vector{\theta})$ is strictly convex.
\end{proof}

\marginalsjoint*
\begin{proof}
    The joint distribution of $\vector{\rv{x}}$ satisfies
    \begin{align*}
        p(d \vector{x}; \matrix{M}) \propto \Tr\Big( \matrix{M} \big(\vector{\psi}_{1\mid 2} (\vector{x}_1) \otimes \vector{\psi}_{1\mid 2} (\vector{x}_1) \big) \Big) \mu_1(d\vector{x}_1 \mid \vector{x}_2) \mu_2(d \vector{x}_2).
    \end{align*}
    The conditional distribution of $\vector{\rv{x}}_1$ given $\vector{\rv{x}}_2 = \vector{x}_2$ is proportional to the joint distribution, so
    \begin{align*}
        P(d \vector{x}_1; \vector{x}_2, \matrix{M}) \propto \Tr\Big( \matrix{M} \big(\vector{\psi}_{1\mid 2} (\vector{x}_1) \otimes \vector{\psi}_{1\mid 2} (\vector{x}_1) \big) \Big) \mu_1(d\vector{x}_1 \mid \vector{x}_2) \propto P_{\mu_1(\cdot \mid \vector{x}_2), \vector{\psi}_{1\mid 2}}(d \vector{x}_1; \matrix{M}).
    \end{align*}
    The marginal distribution of $\vector{\rv{x}}_2$ is obtained by marginalising out the joint distribution with respect to $\vector{x}_1$,
    \begin{align*}
        P_2(d\vector{x}_2;\matrix{M}) &= \frac{1}{\Tr\big( \matrix{M} \matrix{K}_{\mu, \vector{\psi}} \big)} \int_{\mathbb{X}_1} \Tr\Big( \matrix{M} \big(\vector{\psi}_{1\mid 2} (\vector{x}_1) \otimes \vector{\psi}_{1\mid 2} (\vector{x}_1) \big) \Big) \, \mu_1(d\vector{x}_1 \mid \vector{x}_2) \, d\vector{x}_1 \, \mu_2(d \vector{x}_2) \\
        &= \frac{\Tr\big( \matrix{M} \matrix{K}_{\mu_1(\cdot \mid \vector{x}_2), \vector{\psi}_{1 \mid 2}}\big)}{\Tr\big( \matrix{M} \matrix{K}_{\mu, \vector{\psi}} \big)} \mu_2(d\vector{x}_2).
    \end{align*}
\end{proof}

\clearpage

\section{Proofs for \S~\ref{sec:g_fam}}
\subsection{Fisher information calculation for $g$-families}
\label{app:fisher_gfam}
The following simple lemma on $g$-families generalizes some familiar identities of exponential families.
\begin{restatable}{lemma}{lem:identities}
Consider a $g$-family satisfying assumptions \ref{ass:reg_g} and \ref{ass:reg_psi}. Define $\tilde{\vector{\psi}}_{g,\vector{\theta}}(\vector{x}):=\frac{g'\big( \vector{\theta}^\top \vector{\psi}(\vector{x}) \big)}{g\big( \vector{\theta}^\top \vector{\psi}(\vector{x}) \big)}\vector{\psi}(\vector{x})$. Then the following identities hold:
\begin{align*}
\nabla \log z(\vector{\theta})&=\mathbb E\left[\tilde{\vector{\psi}}_{g,\vector{\theta}}(\vector{\rv{x}})\right],\\
\nabla_{\vector{\theta}}\log p(\vector{x}|\vector{\theta})&=\tilde{\vector{\psi}}_{g,\vector{\theta}}(\vector{x})-\mathbb E\left[\tilde{\vector{\psi}}_{g,\vector{\theta}}(\vector{\rv{x}})\right],\\
\matrix{G}(\vector{\theta})&=\mathrm{Cov}\left[\tilde{\vector{\psi}}_{g,\vector{\theta}}(\vector{\rv{x}})\right]\\
&= \mathbb{E}_{p(\cdot \mid \vector{\theta})} \Big[ \frac{\vector{\psi}(\vector{\rv{x}}) \vector{\psi}(\vector{\rv{x}})^\top g'\big( \vector{\theta}^\top \vector{\psi}(\vector{\rv{x}}) \big)^2}{g\big( \vector{\theta}^\top \vector{\psi}(\vector{\rv{x}}) \big)^2 } \Big]  -  \frac{1}{ z(\vector{\theta})^2} \nabla z(\vector{\theta}) \nabla z(\vector{\theta})^\top.
\end{align*}
\end{restatable}

\begin{proof}
Recall
\label{sec:FIM_calc}
\begin{align*}
    z(\vector{\theta}) &= \int_{\mathbb{X}} g\big(\vector{\theta}^\top \vector{\psi}(\vector{x}) \big) \, \mu(d\vector{x}) \\
    \intertext{so} \\
    \nabla z(\vector{\theta}) &=  \int_{\mathbb{X}} \vector{\psi}(\vector{x}) g'\big(\vector{\theta}^\top \vector{\psi}(\vector{x}) \big) \, \mu(d\vector{x}).
\end{align*}
Multiplying the integrand with $\frac{g\big(\vector{\theta}^\top \vector{\psi}(\vector{x}) \big)}{g\big(\vector{\theta}^\top \vector{\psi}(\vector{x}) \big)}$, it follows that
$$
\nabla \log z(\vector{\theta})=\frac{\nabla z(\vector{\theta})}{z(\vector{\theta})}=\int_{\mathbb{X}} \vector{\psi}(\vector{x}) \frac{g'\big(\vector{\theta}^\top \vector{\psi}(\vector{x}) \big)}{g\big(\vector{\theta}^\top \vector{\psi}(\vector{x}) \big)} p(\vector{x}|\vector{\theta})\, \mu(d\vector{x}),
$$
as required. The second identity follows directly from  $\nabla_{\vector{\theta}}\log p(\vector{x}|\vector{\theta})=\nabla_{\vector{\theta}}\log g\big(\vector{\theta}^\top \vector{\psi}(\vector{x}) \big)-\nabla \log z(\vector{\theta})=\tilde{\vector{\psi}}_{g,\vector{\theta}}(\vector{x})-\nabla \log z(\vector{\theta})$. Finally, for Fisher information, we simply plug in the second identity into $\matrix{G}(\vector{\theta})=\mathbb E\left[\nabla_{\vector{\theta}}\log p(\vector{\rv{x}}|\vector{\theta})(\nabla_{\vector{\theta}}\log p(\vector{\rv{x}}|\vector{\theta}))^\top\right]$. 
Alternatively, we can calculate the Fisher information directly as
\begin{align*}
    &\phantom{{}={}} \matrix{G}(\vector{\theta}) \\
    &= \mathbb{E}\Big[ \Big(  \frac{\partial}{\partial \vector{\theta}}\log g\big( \vector{\theta}^\top \vector{\psi}(\vector{\rv{x}}) \big) - \frac{\nabla z(\vector{\theta})}{z(\vector{\theta})} \Big) \Big( \frac{\partial}{\partial \vector{\theta}}\log g\big( \vector{\theta}^\top \vector{\psi}(\vector{\rv{x}}) \big) - \frac{\nabla z(\vector{\theta})}{z(\vector{\theta})} \Big)^\top\Big] \\
    &= \mathbb{E}\Big[ \frac{\vector{\psi}(\vector{\rv{x}}) \vector{\psi}(\vector{\rv{x}})^\top g'\big( \vector{\theta}^\top \vector{\psi}(\vector{\rv{x}}) \big)^2}{g\big( \vector{\theta}^\top \vector{\psi}(\vector{\rv{x}}) \big)^2 } - \frac{\vector{\psi}(\vector{\rv{x}})  g'\big( \vector{\theta}^\top \vector{\psi}(\vector{\rv{x}}) \big)}{g\big( \vector{\theta}^\top \vector{\psi}(\vector{\rv{x}}) \big) } \frac{\nabla z(\vector{\theta})}{z(\vector{\theta})}^\top -  \frac{\nabla z(\vector{\theta})}{z(\vector{\theta})} \frac{\vector{\psi}(\vector{\rv{x}})  g'\big( \vector{\theta}^\top \vector{\psi}(\vector{\rv{x}}) \big)}{g\big( \vector{\theta}^\top \vector{\psi}(\vector{\rv{x}}) \big) }^\top \\
    &\phantom{{}={}} + \frac{1}{z(\vector{\theta})^2} \nabla z(\vector{\theta}) \nabla z(\vector{\theta})^\top \Big] \\
    &= \int_{\mathbb{X}} \frac{\vector{\psi}(\vector{x}) \vector{\psi}(\vector{x})^\top g'\big( \vector{\theta}^\top \vector{\psi}(\vector{x}) \big)^2}{g\big( \vector{\theta}^\top \vector{\psi}(\vector{x}) \big) z(\vector{\theta}) } -  g'\big( \vector{\theta}^\top \vector{\psi}(\vector{x}) \big) \vector{\psi}(\vector{x}) \frac{\nabla z(\vector{\theta})}{z(\vector{\theta})^2}^\top -  g'\big( \vector{\theta}^\top \vector{\psi}(\vector{x}) \big) \frac{\nabla z(\vector{\theta})}{z(\vector{\theta})^2} \vector{\psi}(\vector{x})^\top  \, \mu(d\vector{x}) \\
    &\phantom{{}={}} + \frac{1}{ z(\vector{\theta})^2} \nabla z(\vector{\theta}) \nabla z(\vector{\theta})^\top  \\
    &=\int_{\mathbb{X}} \frac{\vector{\psi}(\vector{x}) \vector{\psi}(\vector{x})^\top g'\big( \vector{\theta}^\top \vector{\psi}(\vector{x}) \big)^2}{g\big( \vector{\theta}^\top \vector{\psi}(\vector{x}) \big) z(\vector{\theta}) } \mu(d\vector{x}) -  \frac{1}{ z(\vector{\theta})^2} \nabla z(\vector{\theta}) \nabla z(\vector{\theta})^\top.
\end{align*}
\end{proof}

We note that in exponential families $\tilde{\vector{\psi}}_{g,\vector{\theta}}(\vector{x})= \vector{\psi}(\vector{x})$ recovering the familiar identities linking the derivative of the log-normalizing constant to the expectation of the sufficient statistic.

Furthermore, note that a $g$-family is orthogonally singular if and only if the random variable $\vector{\theta}^\top \tilde{\vector{\psi}}_{g,\vector{\theta}}(\vector{\rv{x}})$ is $\mu$-almost surely constant. This implies that $g$ must satisfy a particular differential equation, which is further detailed in the proof of Lemma \ref{lemma:even_order}.

\subsection{Fisher information of Gaussian for dimension augmentation}
Let $\rv{a} \sim \mathcal{N}\Big(a \mid \log z(\vector{\theta}), \sigma^2 \Big)$. 
Then we may compute the Fisher information of the distribution $\rv{a}$ as a function of $\vector{\theta}$ either directly or through the reparameterisation of the Fisher information (via the Jacobian of the transform).
Directly, the score function is
\begin{align*}
    \frac{1}{\sigma^2}\big(a - \log z(\vector{\theta}) \big) \frac{\nabla z(\vector{\theta})}{z(\vector{\theta})},
\end{align*}
and therefore the expected outer products of score functions, i.e. the Fisher information, is
\begin{align*}
    \frac{1}{\sigma^{2} z(\vector{\theta})^2} \nabla z(\vector{\theta}) \nabla z(\vector{\theta})^\top.
\end{align*}

\subsection{Other proofs}
\poshomo*
\begin{proof}
The result is a variant of Euler's theorem for homogeneous functions in one dimension. 
The inner product of the parameter $\vector{\theta}$ and the score function is 
\begin{align*}
\vector{\theta}^\top \nabla \log p(\vector{x} \mid \vector{\theta}) &= \frac{\big(\vector{\theta}^\top \vector{\psi}(\vector{x}) \big) g' \big(\vector{\theta}^\top \vector{\psi}(\vector{x}) \big)}{g\big(\vector{\theta}^\top \vector{\psi}(\vector{x}) \big)} - \frac{\int_{\mathbb{X}} \big(\vector{\theta}^\top \vector{\psi}(\vector{x}) \big)g'\big(\vector{\theta}^\top \vector{\psi}(\vector{x}) \big) \, \mu(d\vector{x})}{\int_{\mathbb{X}} g\big(\vector{\theta}^\top \vector{\psi}(\vector{x}) \big) \, \mu(d\vector{x})}. \numberthis \label{eq:score_param}
\end{align*}
First we prove that if~\eqref{eq:param_score_orth}, then $g$ is a positively homogeneous.
The second term in~\eqref{eq:score_param} does not depend on $\vector{x}$, so if~\eqref{eq:param_score_orth}, there exists some function $k(\vector{\theta})$ such that
\begin{align*}
    \frac{ g' \big(\vector{\theta}^\top \vector{\psi}(\vector{x}) \big)}{g\big(\vector{\theta}^\top \vector{\psi}(\vector{x}) \big)} = \big(\vector{\theta}^\top \vector{\psi}(\vector{x}) \big)^{-1} k(\vector{\theta}).
\end{align*}
Let $\mathbb{A}_{\vector{\theta}} = \{\vector{\theta}^\top \vector{\psi}(\vector{x}) \mid \vector{x} \in \mathbb{X}\} \subseteq \mathbb{R}$.
Suppose that for each $\vector{\theta}$, $\mathbb{A}_{\vector{\theta}}$ contains an open set containing $0$. 
Consider some connected subset containing $0$.
Then for all $a$ in this subset, $a g'(a) = k(\vector{\theta}) g(a)$. 
By Euler's theorem for homogeneous functions, this leads to unique solutions of the form
\begin{align*}
    g(a) = \begin{cases} c_1(\vector{\theta}) |a|^{k(\vector{\theta})} \quad a > 0, \\
    c_2(\vector{\theta}) |a|^{k(\vector{\theta})} \quad a \leq 0
    \end{cases}
\end{align*}
for some $c_1(\vector{\theta})$ and $c_2(\vector{\theta})$.
For any two $\vector{\theta}_1$ and $\vector{\theta}_2$, since $\mathbb{A}_{\vector{\theta}_1}$ and $\mathbb{A}_{\vector{\theta}_2}$ is open and contains zero, we must have agreement of the solutions $g$,
\begin{align*}
    c_1(\vector{\theta}_1) |a|^{k(\vector{\theta}_1)} &= c_1(\vector{\theta}_2) |a|^{k(\vector{\theta}_2)}, \\
    c_2(\vector{\theta}_1) |a|^{k(\vector{\theta}_1)} &= c_2(\vector{\theta}_2) |a|^{k(\vector{\theta}_2)}.
\end{align*}
Without loss of generality, choose $g(1) = 1$, implying $c_1(\vector{\theta})$ is constant $1$. 
We then have $k(\vector{\theta}_1)$ is constant, and $c_2(\vector{\theta}_2)$ is constant.
Thus $g$ is of the form
\begin{align*}
    g(a) = \begin{cases} \phantom{c_2}|a|^k, \quad a > 0 \\
    c_2 |a|^k, \quad a \leq 0. \end{cases}
\end{align*}
\iffalse
Fixing any $\vector{\theta}$, $g$ must satisfy
\begin{align*}
    g'(a) / g(a) = k(\vector{\theta}) a^{-1}
\end{align*}
for values of $a$ on $\{ \vector{\theta}^\top \vector{\psi}(\vector{x}) \mid \vector{x} \in \mathbb{X}\} \subseteq \mathbb{R}$.
Solving on the larger space $\mathbb{R}$, we find unique solutions (up to a multiplicative constant) of the form $g(a) \propto a^{k(\vector{\theta})}$. 
Solutions on the domain $\{ \vector{\theta}^\top \vector{\psi}(\vector{x}) \mid \vector{x} \in \mathbb{X}\} $ inherit this uniqueness. 
Since $g$ must be nonnegative and finite, $k(\vector{\theta})$ must be positive and even integer-valued.
Since $g\big(\vector{\theta}^\top \vector{\psi}(\vector{x}) \big)$ is a continuous function of $\vector{\theta}$, the function $k$ must also be continuous. 
The only even integer-valued and continuous function $k$ is a constant even-integer function.
\fi
The reverse direction follows directly by choosing $g$ to be a positively homogeneous function, resulting in a value of $0$ in~\eqref{eq:score_param}.

Note that $c_2 >0$ because $g$ is nonnegative, and $k \geq 2$ because $g$ is twice continuously differentiable. The case $k=0$ is excluded because the Fisher information is exactly $\matrix{0}$.
\end{proof}

\dimaugregular*
\begin{proof}
Choose some nonzero vector $\vector{\theta}'$ and consider the nonnegative quadratic
\begin{align*}
    \vector{\theta}'^\top \matrix{G}_{p_1(\cdot) p_2(\cdot)} (\vector{\theta}) \vector{\theta}' &= \vector{\theta}'^\top \matrix{G}_{p_1(\cdot) } (\vector{\theta}) \vector{\theta}'  + \vector{\theta}'^\top \matrix{G}_{p_2(\cdot)} (\vector{\theta}) \vector{\theta}'.
\end{align*}
    Suppose $\vector{\theta}'$ is proportional to $\vector{\theta}$. Then the first term is exactly zero and the second term is proportional to $\frac{1}{z(\vector{\theta})^2} \vector{\theta}^\top \nabla z(\vector{\theta}) \nabla z(\vector{\theta})^\top \vector{\theta} \propto 1>0$, since $z$ is positively homogeneous. 
    Suppose otherwise that $\vector{\theta}'$ is not proportional to $\vector{\theta}$. Then the $\vector{\theta}'^\top \matrix{G}(\vector{\theta}) \vector{\theta}'$ is the sum of a nonnegative term and some integrals of the form
    \begin{align*}
          \frac{1}{z(\vector{\theta})}  k^2 \int \big| \vector{\theta}'^\top \vector{\psi}(\vector{x}) \big|^2 \big| \vector{\theta}^\top \vector{\psi}(\vector{x}) \big|^{k-2} \mu(d\vector{x}),
    \end{align*}
    and by Assumption~\ref{ass:feature_regularity}, $\vector{\theta}'^\top \vector{\psi}(\vector{x})$ and $\vector{\theta}^\top \vector{\psi}(\vector{x})$ must simultaneously take a nonzero value over a set of nonzero measure.
\end{proof}

\characterisation*

\begin{proof}
    We desire
\begin{align*}
    \nabla^2 \Big( \chi\big(z(\vector{\theta}) \big) \Big) &= c\big(z(\vector{\theta}) \big) \matrix{G}(\vector{\theta}). \numberthis \label{eq:conformal}
\end{align*}
Our approach is to compute both the left hand side and right hand side of~\eqref{eq:conformal}, and then match the coefficients of integrals on either side. 
In order to make clear the coefficients that we are matching, we colour certain factors in the derivation which follows.
We first compute the left hand side,
\begin{align*}
    \nabla^2 \Big( \chi\big(z(\vector{\theta}) \big) \Big) &= \nabla\Bigg( \chi'\big(z(\vector{\theta}) \big) \nabla z(\vector{\theta}) \Bigg) \\
    &= \chi''\big(z(\vector{\theta}) \big)  \nabla z(\vector{\theta})  \nabla z(\vector{\theta})^\top + \chi'\big(z(\vector{\theta}) \big) \nabla^2 z(\vector{\theta}) \\
    &= \chi'\big(z(\vector{\theta}) \big) \int_{\mathbb{X}} \vector{\psi}(\vector{x}) \vector{\psi}(\vector{x})^\top \textcolor{green}{g''\big( \vector{\theta}^\top \vector{\psi}(\vector{x}) \big)} \, d\mu(\vector{x}) + \textcolor{blue}{\chi''\big(z(\vector{\theta}) \big)}  \nabla z(\vector{\theta})  \nabla z(\vector{\theta})^\top \\
    &= \int_{\mathbb{X}} \chi'\big(z(\vector{\theta}) \big)  \vector{\psi}(\vector{x}) \vector{\psi}(\vector{x})^\top \textcolor{green}{g''\big( \vector{\theta}^\top \vector{\psi}(\vector{x}) \big)} + \frac{1}{\mu(\mathbb{X})}\textcolor{blue}{\chi''\big(z(\vector{\theta}) \big)}  \nabla z(\vector{\theta})  \nabla z(\vector{\theta})^\top \, d\mu(\vector{x})  \numberthis \label{eq:conformal_lhs}
\end{align*}
We then compute the right hand side,
\begin{align*}
    c\big(z(\vector{\theta}) \big) \matrix{G}(\vector{\theta}) &= \frac{c\big(z(\vector{\theta}) \big)}{z(\vector{\theta})} \int_{\mathbb{X}} \vector{\psi}(\vector{x}) \vector{\psi}(\vector{x})^\top \frac{g'\big( \vector{\theta}^\top \vector{\psi}(\vector{x}) \big)^2}{g\big( \vector{\theta}^\top \vector{\psi}(\vector{x}) \big)} \, d\mu(\vector{x}) - \frac{c\big(z(\vector{\theta}) \big)}{z(\vector{\theta})^2} \Big( \frac{1}{\sigma^2} - 1 \Big) \nabla z(\vector{\theta}) z(\vector{\theta})^\top,
\end{align*}
Therefore
\begin{align*}
    &c\big(z(\vector{\theta}) \big) \matrix{G}(\vector{\theta}) \\
    &= \frac{c\big(z(\vector{\theta}) \big)}{z(\vector{\theta})} \int_{\mathbb{X}} \vector{\psi}(\vector{x}) \vector{\psi}(\vector{x})^\top \textcolor{green}{\frac{g'\big( \vector{\theta}^\top \vector{\psi}(\vector{x}) \big)^2}{g\big( \vector{\theta}^\top \vector{\psi}(\vector{x}) \big)}} \, d\mu(\vector{x}) + \textcolor{blue}{\frac{c\big(z(\vector{\theta}) \big)}{z(\vector{\theta})^2} \Big( \frac{1}{\sigma^2} - 1 \Big)}\nabla z(\vector{\theta}) z(\vector{\theta})^\top  \\
    &= \int_{\mathbb{X}} \vector{\psi}(\vector{x}) \vector{\psi}(\vector{x})^\top \frac{c\big(z(\vector{\theta}) \big)}{z(\vector{\theta})} \textcolor{green}{\frac{g'\big( \vector{\theta}^\top \vector{\psi}(\vector{x}) \big)^2}{g\big( \vector{\theta}^\top \vector{\psi}(\vector{x}) \big)}}  + \frac{1}{\mu(\mathbb{X})} \textcolor{blue}{\frac{c\big(z(\vector{\theta}) \big)}{z(\vector{\theta})^2} \Big( \frac{1}{\sigma^2} - 1 \Big)}\nabla z(\vector{\theta}) z(\vector{\theta})^\top \, d\mu(\vector{x}). \numberthis \label{eq:conformal_rhs}
\end{align*}
If $\mu$ is allowed to belong to some suitable class of test functions (a subset of all $\mu$ such that $\mu(\mathbb{X}) < \infty$), equality of the integrals implies equality of the integrands in~\eqref{eq:conformal_lhs} and~\eqref{eq:conformal_rhs}, by the fundamental lemma of calculus of variations.
Note that $\nabla z(\vector{\theta})$ is an integral of $\vector{\psi}(\vector{x})$, so if $\vector{\psi}(\vector{x})$ is not proportional to a constant in $\vector{x}$ (as in Assumption~\ref{ass:reg_psi}), then $\nabla z(\vector{\theta})$ and $\vector{\psi}(\vector{x})$ are linearly independent. Hence their rank $1$ outer products $\nabla z(\vector{\theta}) \nabla z(\vector{\theta})^\top$ and $\vector{\psi}(\vector{x}) \vector{\psi}(\vector{x})^\top$ are also linearly independent. 
Matching the coefficients of the second terms on both the left hand side~\eqref{eq:conformal_lhs} and the right hand side~\eqref{eq:conformal_rhs},
\begin{align*}
    \textcolor{blue}{\chi''\big(z(\vector{\theta}) \big)} &= \textcolor{blue}{\frac{c\big(z(\vector{\theta}) \big)}{z(\vector{\theta})^2} \Big( \frac{1}{\sigma^2} - 1 \Big)} \\
    \implies \chi'\big(z \big) &= \Bigg({\underbrace{\Big( \frac{1}{\sigma^2} - 1 \Big)}_{:=r}} \Bigg( \int \frac{c}{z^2} \, dz \Bigg)  + c_0\Bigg),
\end{align*}
where $c_0$ is a constant of integration.
Plugging back into the first terms of~\eqref{eq:conformal_lhs} and~\eqref{eq:conformal_rhs}, we then see that the original condition~\eqref{eq:conformal} becomes
\begin{align*}
    \textcolor{red}{\Bigg(r \Bigg( \int \frac{c(z)}{z^2} \, dz \Bigg) \Bigg|_{z = z(\vector{\theta})} + c_0\Bigg)} & \textcolor{green}{g''\big( \vector{\theta}^\top \vector{\psi}(\vector{x}) \big)}  = \textcolor{red}{\frac{c\big(z(\vector{\theta}) \big)}{z(\vector{\theta})}} \textcolor{green}{\frac{g'\big( \vector{\theta}^\top \vector{\psi}(\vector{x}) \big)^2}{g\big( \vector{\theta}^\top \vector{\psi}(\vector{x}) \big)}}.
\end{align*}
Since $g$ need only be determined up to a proportionality constant, noting that the red factors do not depend on $\vector{x}$, we may take the two green factors as equal and the two red factors as equal.
Thus two independent equations are satisfied, with $r = r_1 r_2$,
\begin{align*}
    r_1  \int \frac{c(z)}{z^2} \, dz + \text{const}  &=\frac{c(z)}{z} \\
    r_2 g''(a) &= \frac{g'(a)^2}{g(a)}.
\end{align*}
It can be verified that
\begin{align*}
    c(z) = c_1 z^{r_1 + 1}  \qquad \text{and} \qquad 
    g(a) = \Big| \frac{r_2-1}{r_2}(a-1) \pm 1 \Big|^{\frac{r_2}{r_2-1}} s(a)
\end{align*}
solve these two independent equations, where $s$ is a positive scaling function depending only on the sign of $a$, provided $r_2/(r_2-1) \geq 2$.
The latter solution for $g$ is unique, and is obtained as follows.
We may transform the ODE to be an ODE in the log unnormalised density.
Choose $l(a) = \log |g(a)|$, noting that $g(a) \geq 0$, and $g(a) =0$ for a countable number of points.
We solve the ODE that follows over disjoint open intervals separated by points such that $g(a) = 0$ (of which there are countably many, by Assumption~\ref{ass:reg_g}). 
(After solving over such regions, we will note there is only one point where $g(a)=0$, and take the solution to be the union of both solutions over disjoint regions).
Under this substitution, we find
\begin{align*}
    g'(a) &= g(a) l'(a), \quad \text{ and}\\
    g''(a) &= g(a) l'(a)^2 + g(a) l''(a), \quad \text{so} \\
    r g(a)^2  \big( l'(a)^2 + l''(a) \big) &= g(a)^2 l'(a)^2 \\
    (r-1) l'(a)^2 + r l''(a)  &= 0.
\end{align*}
This is an example of a Bernoulli differential equation in $l'(a)=h(a)$,
\begin{align*}
    h'(a) &= \frac{1-r}{r} h(a)^2.
\end{align*}
Solving proceeds by direct integration,
\begin{align*}
    \int h^{-2} \, dh =- \frac{1}{h} &= \frac{1-r}{r}a + c_1\\
    h(a) &= \Big( \frac{r-1}{a} - c_1 \Big)^{-1} \\
    l(a) &= \frac{r}{r-1} \log \Big| \frac{r-1}{r}a - c_1 \Big| + c_2 \\
    g(a) &=  \Big| \frac{r-1}{r}a - c_1 \Big|^{\frac{r}{r-1}} \exp(c_2).
\end{align*}
Therefore, noting that $g$ is continuous by assumption, there is only one point $a = r/(r-1)c_1$ such that $g(a) = 0$.
Taking the union of solutions over $a=0$, $a > r/(r-1)c_1$ and $a < r/(r-1)c_1$, and imposing $g(1)=1$, we have
\begin{align*}
        g(a) &=  \Big| \frac{r-1}{r}(a-1)\exp\big((r-1)c_2/r \big) \pm 1 \Big|^{\frac{r}{r-1}} s(a)
\end{align*}
where $s(a)$ is a positive scaling function which depends only on the sign of $\frac{r-1}{r}(a-1)\exp\big((r-1)c_2/r \big) \pm 1 $.
In the case where $r \neq 1$, such a $g$ results in a $g$ family with a positively homogeneous $g$, after accounting for the bias reparameterisation.
Finally note that as $\lim_{r\to 1} g(a) = \exp\big( \pm(a-1) \big)$ gives an exponential family by the limit characterisation of the exponential function.

\iffalse
Valid solutions for $g$ must be nonnegative over all of $\mathbb{R}$.

First suppose $r_2 \neq 1$. 
We must have that $r_2/(r_2-1)$ is an even nonnegative integer, since $g$ is nonnegative and finite.
We may take $c_2=1$ without loss of generality, since $g$ is only defined up to proportionality constant. 
This leads to functions of the form
\begin{align*}
    g(a) = \big(c_3 r_2 + a(1-r_2) \big)^{r_2/(r_2-1)},
\end{align*}
where $r_2/(r_2-1)$ is even. 
Since $g\big( \vector{\theta}^\top \vector{\psi}(\vector{x}) \big) = \big(c_3 r_2 + \theta_n(1-r_2) + (1-r_2) \big(\vector{\theta}_{:n-1}^\top \vector{\psi}(\vector{x})_{:n-1} \big) \big)^{r_2/(r_2-1)},$ we see that the family is an even-order monomial family with reparameterisation $\hat{\vector{\theta}} = \big( (1-r_2) \vector{\theta}_{:n-1}; c_3r_2 + \theta_n(1-r_2) \big)$, since bias terms are included in Assumption~\ref{ass:reg_psi}.

Now suppose $r_2 = 1$. Then through either the limit definition of the exponential function, or by solving the equation $g''(a) = g'(a)^2/g(a)$ with the substitution $h(a) = g'(a)$, unique solutions for $g$ must take the form $g(a) = \exp(c_4 a)$ for some constant $c_4$. 
\fi
\end{proof}

\section{Proofs for \S~\ref{sec:estimation}}
\subsection{Divergence inequalities}
The SH distance, TV distance and KL divergence satisfy
\begin{align*}
    %D_{1 - \sqrt{\cdot}} [p:p'] \leq D_{\frac{1}{2}|\cdot - 1|}[p:p'] \leq \sqrt{2} \sqrt{D_{1 - \sqrt{\cdot}} [p:p']} \leq \sqrt{ D_{(\cdot) \log(\cdot)} [p:p']}. \numberthis \label{eq:pinsker}
    \sqhel [p:p'] \leq \tv [p:p'] \leq \sqrt{2} \sqrt{\sqhel [p:p']} \leq \sqrt{ \kl [p:p']}, \numberthis \label{eq:pinsker}
\end{align*}
see equation (16) of~\citet{sason2016f} for the last inequality, and note that the first two inequalities result from inequalities between the $1$-norm and $2$-norm.
In general a reverse inequality, upper bounding the KL divergence in terms of the TV distance is not possible. 
However, in the restricted setting where $0< \big(\sup \frac{p(\vector{x})}{p'(\vector{x})}\big)^{-1} := r_0 < 1$, an additional reverse inequality may be added to the chain (see~\citet[Theorem 7]{verdu2014total}, or a strengthened result under an additional assumption~\citep[equation 295]{sason2016f}), 
%https://ieeexplore.ieee.org/stamp/stamp.jsp?tp=&arnumber=7552457
%https://arxiv.org/pdf/1508.00335
\begin{align*}
    \sqrt{ \kl [p:p']} \leq \sqrt{\frac{\log (r_0^{-1})}{2(1-r_0)} \tv [p:p']}. \numberthis \label{eq:reverse_pinsker}
\end{align*}

\subsection{Proofs}
\mlemisspecified*
\begin{proof}
We first show that up to some small corrections, the difference between the KL divergence of the finite sample MLE and the infinite sample MLE is asymptotically Gaussian, with zero mean and variance scaling like $N^{-1}$. 
We then show that the small corrections actually account for all of the asymptotic variance, and hence the difference in KL divergences converges to zero.
\begin{theorem}
\label{thm:estimation}
Suppose Assumptions~\ref{ass:strictly_pd},~\ref{ass:well_posed} and~\ref{ass:ratio} hold.
Define
\begin{align*}
    h(\vector{\theta}_\ast) &= \kl\big(q:p(\cdot\mid \vector{\theta}_\ast) \big) - \underbrace{2 \log  \vector{\theta}_\ast^\top \matrix{K}_{\mu, \vector{\psi}} \vector{\theta}_\ast}_{=0}, \quad \text{and}\\ 
    h(\hat{\vector{\theta}}_N) &= \kl\big(q:p(\cdot; \hat{\vector{\theta}}_N) \big) - \underbrace{2 \log  \hat{\vector{\theta}}_N^\top \matrix{K}_{\mu, \vector{\psi}} \hat{\vector{\theta}}_N}_{\text{correction}}.
\end{align*}
Then asymptotically in $N$,
\begin{align*}
        \sqrt{N}\big( h(\hat{\vector{\theta}}_N) - h(\vector{\theta}_\ast) \big) \sim \mathcal{N}( 0, 4 ).
    \end{align*}
\end{theorem}
\begin{proof}
A classical result of~\citet[Theorem 3.2]{white1982maximum} says that the quasi maximum likelihood estimator is still asymptotically normally distributed $\hat{\vector{\theta}}_N \sim \mathcal{N}\big( \vector{\theta}_\ast, \hat{\matrix{K}}\big)$, although unlike in MLE, its mean parameter $\vector{\theta}_\ast$ is no longer the true parameter and its precision matrix $\hat{\matrix{K}}$ is no longer governed by the Fisher information.
All of the required regularity conditions for~\citet[Theorem 3.2]{white1982maximum} hold. 
The mean parameter is~\eqref{eq:kl_minimiser}, the minimiser of the KL divergence from the dimension augmented squared family model to the dimension augmented target distribution.
Since the first term of the objective~\eqref{eq:kl_minimiser} does not depend on the value of $ \vector{\theta}'^\top \matrix{K}_{\mu, \vector{\psi}} \vector{\theta}'$, we must have that ${\vector{\theta}_\ast}^\top \matrix{K}_{\mu, \vector{\psi}} {\vector{\theta}_\ast}=1$ by minimising the second term of the objective~\eqref{eq:kl_minimiser}.
The covariance parameter $\hat{\matrix{K}}$ is formed by ``sandwiching'' the expected Hessian of the log likelihood against the inverse expected outer products of the score functions~\citep[Theorem 3.2]{white1982maximum},
\begin{align*}
    \hat{\matrix{K}}(\vector{\theta}_\ast) &= \matrix{A}(\vector{\theta}_\ast)^{-1} \matrix{B}(\vector{\theta}_\ast) \matrix{A}(\vector{\theta}_\ast)^{-1},
\end{align*}
where
\begin{align*}
   \matrix{A}(\vector{\theta}) &= \mathbb{E}_q\big[ \frac{\partial^2}{\partial \vector{\theta} \partial \vector{\theta}^\top} \log p(\hat{\vector{\rv{x}}} \mid \vector{\theta}) \mathcal{N}(a\mid \log \vector{\theta}^\top \matrix{K}_{\mu, \vector{\psi}} \vector{\theta} , 1) \big] \\
   \matrix{B}(\vector{\theta}) &= \mathbb{E}_q\big[ \frac{\partial}{\partial \vector{\theta} } \log p(\hat{\vector{\rv{x}}} \mid \vector{\theta}) \frac{\partial}{\partial \vector{\theta} } \mathcal{N}(a\mid \log \vector{\theta}^\top \matrix{K}_{\mu, \vector{\psi}} \vector{\theta} , 1) \log p(\hat{\vector{\rv{x}}} \mid \vector{\theta}) \mathcal{N}(a\mid \log \vector{\theta}^\top \matrix{K}_{\mu, \vector{\psi}} \vector{\theta} , 1)^\top\big],
\end{align*}
and expectations are taken with respect to the true density $q$.
In this case, it can be shown (see Appendix~\ref{app:misspecified_calc}) that
\begin{align*}
    \matrix{A}(\vector{\theta}_\ast) &= -2 \matrix{K}_{r(\cdot) \mu, \vector{\psi}} - 2\matrix{K}_{\mu, \vector{\psi}} \qquad \text{and} \\
    \matrix{B}(\vector{\theta}_\ast) &=  4 \big( \matrix{K}_{r(\cdot) \mu, \vector{\psi}}  -  \matrix{K}_{r(\cdot) \mu, \vector{\psi}} \vector{\theta}_\ast {\vector{\theta}_\ast}^\top \matrix{K}_{\mu, \vector{\psi}} - \matrix{K}_{ \mu, \vector{\psi}} \vector{\theta}_\ast {\vector{\theta}_\ast}^\top \matrix{K}_{r(\cdot) \mu, \vector{\psi}} + 2 \matrix{K}_{\mu, \vector{\psi}} \vector{\theta}_\ast {\vector{\theta}_\ast}^\top \matrix{K}_{\mu, \vector{\psi}} \big).
\end{align*}
Note that $\matrix{K}_{r(\cdot) \mu, \vector{\psi}}$ is positive semidefinite, and therefore under Assumption~\ref{ass:strictly_pd}, ${\matrix{K}_{r(\cdot) \mu, \vector{\psi}} + \matrix{K}_{\mu, \vector{\psi}}}$ is strictly positive definite, so $\matrix{A}(\vector{\theta}_\ast)$ is invertible.

The delta method can be used to transfer asymptotic normality of the parameter to asymptotic normality of a function of the parameter. 
The delta method says that
    \begin{align*}
        \sqrt{N}\big( h(\hat{\vector{\theta}}_N) - h(\vector{\theta}_\ast) \big) \sim \mathcal{N}\Big( 0, \nabla h(\vector{\theta}_\ast)^\top \hat{\matrix{K}} \nabla h(\vector{\theta}_\ast) \Big),
    \end{align*}
    for any differentiable function $h$ with a nonzero derivative. 
    We set $h$ to be the KL divergence minus $2\log  \vector{\theta}^\top \matrix{K}_{\mu, \vector{\psi}} \vector{\theta} $ and compute the required derivative,
    \begin{align*}
        h(\vector{\theta}) &= \kl\big(q:p(\cdot\mid \vector{\theta}) \big) - 2 \log  \vector{\theta}^\top \matrix{K}_{\mu, \vector{\psi}} \vector{\theta}= \mathbb{E}_q[ \log q - \log \vector{\theta}^\top \vector{\psi}(\vector{x}) \vector{\psi}(\vector{x})^\top \vector{\theta}] - \log  \vector{\theta}^\top \matrix{K}_{\mu, \vector{\psi}} \vector{\theta} \\
        \nabla h(\vector{\theta}) &= \frac{-2}{z(\vector{\theta})} \Big(  \matrix{K}_{\mu, \vector{\psi}} + \matrix{K}_{r(\cdot) \mu, \vector{\psi}} \Big) \vector{\theta} \\
        \nabla h(\vector{\theta}_\ast) &= -2 \Big(  \matrix{K}_{\mu, \vector{\psi}} + \matrix{K}_{r(\cdot) \mu, \vector{\psi}} \Big) \vector{\theta}_\ast.
    \end{align*}
    Noting that factors in $\nabla h(\vector{\theta}_\ast)$ cancel with inverse factors in $\matrix{A}(\vector{\theta}_\ast)^{-1}$ because the correction has been chosen conveniently, the asymptotic covariance $\nabla h(\vector{\theta}_\ast)^\top \hat{\matrix{K}} \nabla h(\vector{\theta}_\ast)$ is
    \begin{align*}
         4 {\vector{\theta}_\ast}^\top  \big( \matrix{K}_{r(\cdot) \mu, \vector{\psi}}  -  \matrix{K}_{r(\cdot) \mu, \vector{\psi}} \vector{\theta}_\ast {\vector{\theta}_\ast}^\top \matrix{K}_{\mu, \vector{\psi}} - \matrix{K}_{ \mu, \vector{\psi}} \vector{\theta}_\ast {\vector{\theta}_\ast}^\top \matrix{K}_{r(\cdot) \mu, \vector{\psi}} + 2 \matrix{K}_{\mu, \vector{\psi}} \vector{\theta}_\ast {\vector{\theta}_\ast}^\top \matrix{K}_{\mu, \vector{\psi}} \big) {\vector{\theta}_\ast}.
    \end{align*}
    Since ${\vector{\theta}_\ast}^\top \matrix{K}_{\mu, \vector{\psi}} {\vector{\theta}_\ast}=1$, this simplifies as
    \begin{align*}
        &\phantom{{}={}} 4   {\vector{\theta}_\ast}^\top \matrix{K}_{r(\cdot) \mu, \vector{\psi}}\vector{\theta}_\ast \\
        &= 4  z(\vector{\theta}_\ast) {\vector{\theta}_\ast}^\top \Bigg( \int_{\mathbb{X}} \vector{\psi}(\vector{x})\vector{\psi}(\vector{x})^\top \frac{q(\vector{x})}{{\vector{\theta}_\ast}^\top \vector{\psi}(\vector{x})\vector{\psi}(\vector{x})^\top \vector{\theta}_\ast} \, d\mu(\vector{x}) \Bigg) \vector{\theta}_\ast \\
        &= 4.
    \end{align*}
\end{proof}
Now observe that for any $N$, $ \kl\big(q:p(\cdot; \hat{\vector{\theta}}_N) \big)$ and $2 \log  \hat{\vector{\theta}}_N^\top \matrix{K}_{\mu, \vector{\psi}} \hat{\vector{\theta}}_N$ are statistically independent. 
To see this, note that~\eqref{eq:qmle} can be solved by solving two independent optimisation problems --- one depending only on $a_i$, and one depending only on $\vector{x}_i$ --- because of absolute homogeneity. 
The asymptotic variance of $\sqrt{N} 2 \log  \hat{\vector{\theta}}_N^\top \matrix{K}_{\mu, \vector{\psi}} \hat{\vector{\theta}}_N$ is $4$, because the Fisher information of a Gaussian with unit variance and unknown mean is $1$, and $2 \log  \hat{\vector{\theta}}_N^\top \matrix{K}_{\mu, \vector{\psi}} \hat{\vector{\theta}}_N$ is twice the unknown mean. 
Hence by Theorem~\ref{thm:estimation},
\begin{align*}
     \sqrt{N} \Big| \kl\big(q:p(\cdot; \hat{\vector{\theta}}_N) \big) - \kl\big(q:p(\cdot\mid \vector{\theta}_\ast) \big) \Big| + \sqrt{N} 2 \log  \hat{\vector{\theta}}_N^\top \matrix{K}_{\mu, \vector{\psi}} \hat{\vector{\theta}}_N \sim \mathcal{N}(0, 4),
\end{align*}
and the two random variables on the left hand side are independent for all $N$. 
The asymptotic variance of the second term on the left hand side is $4$, and the asymptotic mean is $0$.
Therefore, the asymptotic mean and variance of the first term on the left hand side are both zero, implying convergence in probability to zero.
\end{proof}

\universalapprox*
\begin{proof}
Our approach is to bound the KL divergence of the optimal model using the universal approximation property.
In order to make use of the universal approximation property, we need to bound various statistical divergences in terms of the squared $L^2$ distance.
We will make use of the bounds~\eqref{eq:pinsker} and~\eqref{eq:reverse_pinsker}, and also a bound relating the squared Hellinger distance and the squared $L^2$ distance, as follows.
\begin{lemma}
\label{lemma:hel_bound}
    The squared Hellinger distance between a member of a squared family and an arbitrary probability denisty is bounded by the squared $L^2$ distance between $\vector{\theta}^\top \vector{\psi}(\cdot)/\sqrt{z(\vector{\theta})}$ and $\sqrt{q(\cdot)}$.
\end{lemma}
\begin{proof}
    \begin{align*}
        2 \sqhel [p: q]&= \int_{\mathbb{X}} \big(\sqrt{z}^{-1} \big| \vector{\theta}^\top \vector{\psi}(\vector{x}) \big| - \sqrt{q(\vector{x})} \big)^2 \, d\mu(\vector{x}) && \\
        &=\int_{\mathbb{X}} \Big(\sqrt{z}^{-1}  \text{sgn}\big(\vector{\theta}^\top \vector{\psi}(\vector{x}) \big) \vector{\theta}^\top \vector{\psi}(\vector{x})  - \sqrt{q(\vector{x})} \Big)^2 \, d\mu(\vector{x}) && \\
        &= \int_{\mathbb{X}} \Big( \sqrt{z}^{-1} \vector{\theta}^\top \vector{\psi}(\vector{x})  - \text{sgn}\big(\vector{\theta}^\top \vector{\psi}(\vector{x}) \big) \sqrt{q(\vector{x})} \Big)^2 \, d\mu(\vector{x}) && \\
        &\leq \int_{\mathbb{X}} \Big( \sqrt{z}^{-1} \vector{\theta}^\top \vector{\psi}(\vector{x})  -  \sqrt{q(\vector{x})} \Big)^2 \, d\mu(\vector{x}) && \\
        &= 2 \sqltwo\Big(\sqrt{z}^{-1} \vector{\theta}^\top \vector{\psi}(\vector{x}), \sqrt{q(\vector{x})}\Big) &&,
    \end{align*}
where the inequality holds because $\sqrt{q}$ is positive and the evaluation of the $\text{sgn}$ function always aligns the sign of $\sqrt{q(x)}$ with $\vector{\theta}^\top \vector{\psi}(\vector{x})$.
\end{proof}
We now construct an artificial KL divergence minimiser which we can bound, by considering mixtures of both $p(\vector{x}_i \mid \vector{\theta})$ and $q(\vector{x})$ with some arbitrarily small constant mixture component.
Suppose $1< \mu(\mathbb{X}) < \infty$. 
For some fixed $0 < \epsilon < 1$, we consider 
\begin{align*}
p^{\epsilon}(\vector{x}_i \mid \vector{\theta}) &= (1-\epsilon) \frac{\vector{\theta}^\top\vector{\psi}(\vector{x})\vector{\psi}(\vector{x})^\top \vector{\theta}}{\vector{\theta}^\top \matrix{K}_{\mu, \vector{\psi}} \vector{\theta}} + \epsilon/\mu(\mathbb{X}), \qquad q^\epsilon(\vector{x}) = (1-\epsilon) q(\vector{x}) + \epsilon/\mu(\mathbb{X}) \\
    \vector{\theta}_\ast^\epsilon &= \argmin_{\theta_1' \geq \varepsilon > 0}  \kl [q(\vector{x})  : p^{\epsilon}(\vector{x}_i \mid \vector{\theta})  ] + \kl [\mathcal{N}(a\mid 0, 1)  : \mathcal{N}(a\mid \log  \vector{\theta}^\top \matrix{K}_{\mu, \vector{\psi}} \vector{\theta} , 1)] \\
    &= \argmin_{\substack{\theta_1' \geq \varepsilon > 0 \\  \vector{\theta}^\top \matrix{K}_{\mu, \vector{\psi}} \vector{\theta} = 1}}  \kl [q(\vector{x})  : p^{\epsilon}(\vector{x}_i \mid \vector{\theta})  ].
\end{align*}

\begin{lemma}
\label{lemma:artificial_kl1}
Suppose Assumption~\ref{ass:universal_approx} holds, with $\sqrt{q} \in \mathcal{F}$.
Assume $q(\vector{x}) < q_{\text{max}} < \infty$ and $1\leq\mu(\mathbb{X})<\infty$.
Then there exists some constant $C$ depending on $\epsilon$ such that
\begin{align*}
    \sqrt{\kl\Big(q^\epsilon: p^{\epsilon}(\cdot \mid \vector{\theta}_\ast^\epsilon) \Big)} \leq C/n^{1/4}.
\end{align*}
\end{lemma}
\begin{proof}
We will use the inequality~\eqref{eq:reverse_pinsker}, which requires a parameter $r_0(\vector{\theta}) \in (0, 1)$, where
\begin{align*}
    r_0(\vector{\theta}) = \inf_{\vector{x} \in \mathbb{X}} \frac{ p^{\epsilon}(\vector{x}\mid \vector{\theta})}{q^\epsilon(\vector{x})}.
\end{align*}
Firstly, $r_0(\vector{\theta})$ is strictly greater than zero, since
\begin{align*}
    r_0(\vector{\theta}) &\geq \epsilon/\big(\mu(\mathbb{X}) (1-\epsilon) q_{\text{max}} + \epsilon\big) > 0.
\end{align*}
Secondly, $r_0$ is strictly less than $1$, since
\begin{align*}
    \big(1 - \epsilon + \epsilon/\mu(\mathbb{X}) \big)\vector{\theta}^\top \matrix{K}_{\mu, \vector{\psi}} \vector{\theta} &= \int_{\mathbb{X}} \frac{(1-\epsilon) \vector{\theta}^\top \vector{\psi}(\vector{x})\vector{\psi}(\vector{x})^\top \vector{\theta} + \big(\epsilon/\mu(\mathbb{X})\big) \vector{\theta}^\top \matrix{K}_{\mu, \vector{\psi}} \vector{\theta} }{q^\epsilon(\vector{x})} q^\epsilon(\vector{x}) \, \mu(d\vector{x}) \\
    &> \inf_{\vector{x} \in \mathbb{X}} \frac{(1-\epsilon) \vector{\theta}^\top \vector{\psi}(\vector{x})\vector{\psi}(\vector{x})^\top \vector{\theta} + \big(\epsilon/\mu(\mathbb{X})\big) \vector{\theta}^\top \matrix{K}_{\mu, \vector{\psi}} \vector{\theta} }{q^\epsilon(\vector{x})} \underbrace{\int_{\mathbb{X}}  q^\epsilon(\vector{x}) \, \mu(d\vector{x})}_{=1}, \\
    \text{so} \qquad r_0(\vector{\theta}) &= \inf_{\vector{x}}  \frac{(1-\epsilon)\vector{\theta}^\top\vector{\psi}(\vector{x})\vector{\psi}(\vector{x})^\top \vector{\theta}+\big(\epsilon/\mu(\mathbb{X})\big) \vector{\theta}^\top \matrix{K}_{\mu, \vector{\psi}} \vector{\theta}}{q^\epsilon(\vector{x})\vector{\theta}^\top \matrix{K}_{\mu, \vector{\psi}} \vector{\theta}} \\
    &< 1 - \epsilon + \epsilon/\mu(\mathbb{X})  \\
    &\leq 1.
\end{align*}
The inequalities~\eqref{eq:pinsker} and~\eqref{eq:reverse_pinsker} say that for any $\vector{\theta}$,
\begin{alignat*}{2}
    &\phantom{{}={}} \kl\Big(q^\epsilon: p^{\epsilon}(\cdot \mid \vector{\theta}_\ast^\epsilon) \Big) && \\
    &\leq \kl\Big(q^\epsilon: p^{\epsilon}(\cdot \mid \vector{\theta}) \Big) && \text{(def. of $\vector{\theta}_\ast^\epsilon$)}\\ 
    &\leq \frac{\log (r_0^{-1})}{2(1-r_0)} \tv \Big( q^\epsilon: p^{\epsilon}(\cdot \mid \vector{\theta}) \Big) && \text{(by~\eqref{eq:reverse_pinsker})}\\
    &= \frac{\log (r_0^{-1})}{2(1-r_0)} (1-\epsilon) \tv \Big( q: p(\cdot \mid \vector{\theta}) \Big) && \text{(def. of $\tv$)} \\
    &\leq \frac{\log (r_0^{-1})}{2(1-r_0)} (1-\epsilon) \sqrt{2 \sqhel \Big( q: p(\cdot \mid \vector{\theta}) \Big)} && \text{(by~\eqref{eq:pinsker})}\\
    &\leq \frac{\log (r_0^{-1})}{2(1-r_0)} (1-\epsilon) \sqrt{2 \sqltwo \Big( \sqrt{q(\vector{x})}: \sqrt{z}^{-1} \vector{\theta}^\top \vector{\psi}(\vector{x}) \Big)}\qquad  && \text{(by Lemma~\ref{lemma:hel_bound})}.
\end{alignat*}
We choose $\vector{\theta}$ to be the minimiser of the squared $L^2$ factor in the square root, over the constraint set $\Vert \vector{\theta} \Vert_2^2/\vector{\theta}^\top \matrix{K}_{\mu, \vector{\psi}} \vector{\theta} \leq c$.
Letting $\vector{\theta}' = z(\vector{\theta})^{-1/2}{\vector{\theta}}$, and noting that the constraint transforms as $\Vert \vector{\theta}' \Vert^2 = (\vector{\theta}^\top \matrix{K}_{\mu, \vector{\psi}} \vector{\theta})^{-1} \Vert \vector{\theta} \Vert_2^2 \leq c $, we have
\begin{align*}
    \sqltwo \Big( \sqrt{q(\vector{x})}: \sqrt{z}^{-1} \vector{\theta}^\top \vector{\psi}(\vector{x}) \Big) = \min_{\Vert \vector{\theta}' \Vert_2^2 \leq c } \sqltwo({\vector{\theta}'}^\top \vector{\psi}(\vector{x}): \sqrt{q(\vector{x})}).
\end{align*}
Finally, since the model has a universal approximation property according to Assumption~\ref{ass:universal_approx}, the minimising $L^2$ distance is bounded and there exists a constant $C$ such that 
\begin{align*}
    \kl\Big(q^\epsilon: p^{\epsilon}(\cdot \mid \vector{\theta}_\ast^\epsilon) \Big) \leq C/n^{1/2}.
\end{align*}
\end{proof}

Finally, we bound the artificial KL divergence minimiser in terms of the true KL divergence minimiser of interest.
\begin{lemma}
\label{lemma:artificial_kl2}
    \begin{align*}
    \kl\big(q: p(\cdot\mid \vector{\theta}_\ast) \big)  \leq (1-\epsilon)^{-1} \frac{\log (r_0^{-1})}{2(1-r_0)}  \sqrt{\kl\Big(q^\epsilon: p^{\epsilon}(\cdot \mid \vector{\theta}_\ast^\epsilon) \Big)} 
\end{align*}
\end{lemma}

\begin{proof}
    \begin{alignat*}{2}
        &\phantom{{}={}}\kl\Big(q^\epsilon: p^{\epsilon}(\cdot \mid \vector{\theta}_\ast^\epsilon) \Big) \\
        &\geq \tv \Big(q^\epsilon: p^{\epsilon}(\cdot \mid \vector{\theta}_\ast^\epsilon) \Big)^2  && \text{(by~\eqref{eq:pinsker})} \\
        &= \Big( \frac{2(1-r_0)}{\log (r_0^{-1})} \Big)^2 \Big(\frac{\log (r_0^{-1})}{2(1-r_0)}\Big)^2\tv \Big(q^\epsilon: p^{\epsilon}(\cdot \mid \vector{\theta}_\ast^\epsilon) \Big)^2  && \\
        &= \Big( \frac{2(1-r_0)}{\log (r_0^{-1})} \Big)^2 \Big(\frac{\log (r_0^{-1})}{2(1-r_0)}\Big)^2 (1-\epsilon)^2 \tv \Big(q: p(\cdot \mid \vector{\theta}_\ast^\epsilon) \Big)^2 \qquad && \text{(def. of $\tv$)}\\
        &\geq \Big( \frac{2(1-r_0)}{\log (r_0^{-1})} \Big)^2  (1-\epsilon)^2 \kl \Big(q: p(\cdot \mid \vector{\theta}_\ast^\epsilon) \Big)^2  && \text{(by~\eqref{eq:reverse_pinsker})}\\
        &\geq \Big( \frac{2(1-r_0)}{\log (r_0^{-1})} \Big)^2  (1-\epsilon)^2 \kl \Big(q: p(\cdot \mid \vector{\theta}_\ast) \Big)^2  && \text{(def. of $\vector{\theta}_\ast$)}. 
    \end{alignat*}
\end{proof}
Combining Lemmas~\ref{lemma:artificial_kl1} and~\ref{lemma:artificial_kl2} with Theorem~\ref{thm:estimation}, we have the desired result, with a constant depending on $\epsilon$ (the optimal constant may in principle be investigated, but we do not pursue this here).
\end{proof}

\subsection{Calculations of Fisher information-like matrices}
\label{app:misspecified_calc}
The log likelihood is
\begin{align*}
    &\phantom{{}={}} \log p(\vector{x} \mid \vector{v}) \mathcal{N}(a\mid \log (\vector{v}^\top \matrix{K}_{\mu, \vector{\psi}} \vector{v}) , 1) \\
    &= \log (\vector{v}^\top \vector{\psi}(\vector{x}) \vector{\psi}(\vector{x}) ^\top \vector{v}) - \log (\vector{v}^\top \matrix{K}_{\mu, \vector{\psi}} \vector{v} ) - \frac{1}{2}(a - \log (\vector{v}^\top \matrix{K}_{\mu, \vector{\psi}} \vector{v}))^2 + \text{const}.
\end{align*}
The gradient of the log likelihood is
\begin{align*}
    &\phantom{{}={}} \nabla \log p(\vector{x} \mid \vector{v}) \mathcal{N}(a\mid \log (\vector{v}^\top \matrix{K}_{\mu, \vector{\psi}} \vector{v}) , 1) \\
    &= \frac{2 \vector{\psi}(\vector{x}) \vector{\psi}(\vector{x})^\top \vector{v}}{\vector{v}^\top \vector{\psi}(\vector{x}) \vector{\psi}(\vector{x})^\top \vector{v} } - \frac{2 \matrix{K}_{\mu, \vector{\psi}}\vector{v}}{\vector{v}^\top \matrix{K}_{\mu, \vector{\psi}} \vector{v} } + (a - \log (\vector{v}^\top \matrix{K}_{\mu, \vector{\psi}} \vector{v})) \Big( \frac{2 \matrix{K}_{\mu, \vector{\psi}} \vector{v}}{\vector{v}^\top \matrix{K}_{\mu, \vector{\psi}} \vector{v} }\Big).
\end{align*}
The Hessian of the log likelihood is
\begin{align*}
    &\phantom{{}={}} \nabla^2 \log p(\vector{x} \mid \vector{v}) \mathcal{N}(a\mid \log (\vector{v}^\top \matrix{K}_{\mu, \vector{\psi}} \vector{v}) , 1) \\
    &= \frac{2 \vector{\psi}(\vector{x}) \vector{\psi}(\vector{x})^\top }{\vector{v}^\top \vector{\psi}(\vector{x}) \vector{\psi}(\vector{x})^\top \vector{v} } 
    - \frac{4 \vector{\psi}(\vector{x}) \vector{\psi}(\vector{x})^\top\vector{v} \vector{v}^\top \vector{\psi}(\vector{x}) \vector{\psi}(\vector{x})^\top }{(\vector{v}^\top \vector{\psi}(\vector{x}) \vector{\psi}(\vector{x})^\top \vector{v} )^2}
    - \frac{2 \matrix{K}_{\mu, \vector{\psi}}}{\vector{v}^\top \matrix{K}_{\mu, \vector{\psi}} \vector{v} }
    + \frac{4 \matrix{K}_{\mu, \vector{\psi}}\vector{v}\vector{v}^\top \matrix{K}_{\mu, \vector{\psi}}}{(\vector{v}^\top \matrix{K}_{\mu, \vector{\psi}} \vector{v} )^2} \\
    &\phantom{{}={}}- \frac{4 \matrix{K}_{\mu, \vector{\psi}}\vector{v}\vector{v}^\top \matrix{K}_{\mu, \vector{\psi}}}{(\vector{v}^\top \matrix{K}_{\mu, \vector{\psi}} \vector{v} )^2} 
    +  (a - \log (\vector{v}^\top \matrix{K}_{\mu, \vector{\psi}} \vector{v})) \Bigg( \frac{2 \matrix{K}_{\mu, \vector{\psi}}}{\vector{v}^\top \matrix{K}_{\mu, \vector{\psi}} \vector{v} } -  \frac{4 \matrix{K}_{\mu, \vector{\psi}}\vector{v}\vector{v}^\top \matrix{K}_{\mu, \vector{\psi}}}{(\vector{v}^\top \matrix{K}_{\mu, \vector{\psi}} \vector{v} )^2} \Bigg).
\end{align*}
At $\vector{v}^\ast$, we must have ${\vector{v}^\ast}^\top \matrix{K}_{\mu, \vector{\psi}} {\vector{v}^\ast}  =1$. 
The matrix $\matrix{A}$ is the expected Hessian with respect to the target $q$ evaluated at $\vector{v}^\ast$, yielding
\begin{align*}
    \matrix{A}(\vector{v}^\ast) &= -  2 \matrix{K}_{r(\cdot) \mu, \vector{\psi}} -  2 \matrix{K}_{\mu, \vector{\psi}}.
\end{align*}
The matrix $\matrix{B}$ is the expected outer products of the gradients with respect to the target $q$ evaluated at $\vector{v}^\ast$.
The outer products are the sum of some terms with zero expectation and
\begin{align*}
    &\frac{4 \vector{\psi}(\vector{x}) \vector{\psi}(\vector{x})^\top \vector{v}\vector{v}^\top\vector{\psi}(\vector{x}) \vector{\psi}(\vector{x})^\top}{(\vector{v}^\top \vector{\psi}(\vector{x}) \vector{\psi}(\vector{x})^\top \vector{v} )^2} 
    - \frac{4 \vector{\psi}(\vector{x}) \vector{\psi}(\vector{x})^\top \vector{v}\vector{v}^\top\matrix{K}_{\mu, \vector{\psi}}}{(\vector{v}^\top \vector{\psi}(\vector{x}) \vector{\psi}(\vector{x})^\top \vector{v} )\vector{v}^\top \matrix{K}_{\mu, \vector{\psi}} \vector{v} }  \\
    &- \frac{4  \matrix{K}_{\mu, \vector{\psi}}  \vector{v}\vector{v}^\top\vector{\psi}(\vector{x}) \vector{\psi}(\vector{x})^\top}{(\vector{v}^\top \vector{\psi}(\vector{x}) \vector{\psi}(\vector{x})^\top \vector{v} )\vector{v}^\top \matrix{K}_{\mu, \vector{\psi}} \vector{v} } 
    + 4 \frac{\matrix{K}_{\mu, \vector{\psi}}  \vector{v}\vector{v}^\top \matrix{K}_{\mu, \vector{\psi}}}{(\vector{v}^\top\matrix{K}_{\mu, \vector{\psi}}\vector{v})^2} \\
    &+(a - \log (\vector{v}^\top \matrix{K}_{\mu, \vector{\psi}} \vector{v}))^2 4 \frac{\matrix{K}_{\mu, \vector{\psi}}  \vector{v}\vector{v}^\top \matrix{K}_{\mu, \vector{\psi}}}{(\vector{v}^\top\matrix{K}_{\mu, \vector{\psi}}\vector{v})^2}.
\end{align*}
The matrix $\matrix{B}$ is therefore
\begin{align*}
    \matrix{B}(\vector{v}^\ast) &= 4 \matrix{K}_{r(\cdot) \mu, \vector{\psi}} 
    - 4  \matrix{K}_{r(\cdot) \mu, \vector{\psi}}  {\vector{v}^\ast} {\vector{v}^\ast}^\top \matrix{K}_{\mu, \vector{\psi}} 
    - 4    \matrix{K}_{\mu, \vector{\psi}}  {\vector{v}^\ast} {\vector{v}^\ast}^\top \matrix{K}_{r(\cdot) \mu, \vector{\psi}} 
    + 8    \matrix{K}_{\mu, \vector{\psi}}  {\vector{v}^\ast} {\vector{v}^\ast}^\top \matrix{K}_{\mu, \vector{\psi}}
\end{align*}

\end{document}